%% file: main.tex
\documentclass[oneside]{article}

\input{./packages.tex}
\input{./macros.tex}

\title{\bf \huge Working paper: Towards a Category-theoretic Comparative Framework for Artificial General Intelligence
}

\author{
  Pablo de los Riscos, Fernando  Corbacho \\
  Cognodata R+D \&\\
  Universidad Autonoma de Madrid \\
  \texttt{\{pablo.delosriscos, fernando.corbacho | @cognodata.com\}} \\
  \And
  Michael A. Arbib \\
  University of California \\
  San Diego\\
  \texttt{arbib@usc.edu} \\
}

\begin{document}
\maketitle
\setcounter{tocdepth}{2}

\begin{abstract}

AGI has become the "Holy Grail" of AI with the promise of human level intelligence and the major Tech companies around the world are investing unprecedented amounts of resources in its pursuit. Yet, there does not exist a  single formal definition and only some empirical AGI benchmarking frameworks currently exist.  
The main purpose of this paper is to develop a general, algebraic and category-theoretic framework for describing, comparing and analysing different possible AGI architectures.
Thus, this Category-theoretic formalization would also allow to compare different possible candidate AGI architectures, such as,  Reinforcement Learning (RL), Universal AI, Active Inference, Causal Reinforcement Learning, Schema-based Learning (SBL), etc. It will allow to unambiguously expose their commonalities and differences, and what is even more important, expose areas for future research. 
From the applied Category-theoretic point of view, we take as inspiration "Machines in a Category" to provide a modern view of "AGI Architectures in a Category".
More specifically, this first position paper provides, on one hand, a first exercise on RL, Causal RL and SBL Architectures in a Category, and on the other hand, it is a first step on a broader research program that seeks to provide a unified formal foundation for AGI systems, integrating architectural structure, informational organization, semantic/agent realization, agent–environment interaction, behavioural development over time, and the empirical evaluation of properties. This framework is also intended to support the definition of architectural properties, both syntactic and informational, as well as semantic properties of agents and their assessment in environments with explicitly characterized features.
In the present paper, however, we restrict attention to the architectural layer and its categorical formalization, offering only brief formalization of agents implementation given an architecture.
The paper ends with a proposed Work Plan to achieve the ultimate objective of constructing a general Category-theoretic comparative Framework for a very broad spectrum of AGI Agent Architectures.
We claim that Category Theory and AGI will have a very \textit{symbiotic relation}. That is, AGI will immensely benefit from a Category-theoretic general formalization, while, at the same time, Category Theory will become the front line mathematical paradigm thanks to the extremely wide interest in AGI.

\end{abstract}


\newpage
\tableofcontents

\input{Parts/1_Introduction}

\input{Parts/2_Related_Work}

\input{Parts/3_ArchAgents_Category}
\input{Parts/4_Agents_Category}

\input{Parts/5_Properties}
\input{Parts/6_Architectures_Cases_Studies}

\input{Parts/7_Agents_implementations}
\input{Parts/Work_in_Progress}

\input{Parts/Conclusion}

\newpage

\appendix
\input{Parts/Framework_Theorems}



\twocolumn
\nocite{*}
\bibliographystyle{unsrt}  
\bibliography{bib}
\onecolumn

\end{document}

%% file: packages.tex
\usepackage{style}

\usepackage[utf8]{inputenc} 
\usepackage[T1]{fontenc}    
\usepackage{hyperref}       
\usepackage{url}            
\usepackage{booktabs}       
\usepackage{amsfonts}       
\usepackage{nicefrac}       
\usepackage{microtype}      
\usepackage{lipsum}
\usepackage{fancyhdr}       
\usepackage{graphicx}       
\graphicspath{{media/}}     

\usepackage[mono=false]{libertine} 

\usepackage{luatex85}

\usepackage{iftex,listings,fancyvrb,longtable,xcolor,soul,setspace,babel}
\usepackage{latexsym}
\usepackage{capt-of}
\usepackage{booktabs}
\usepackage[most]{tcolorbox}
\usepackage{times}
\usepackage{tikz}
\usepackage{amsmath,amsthm,amssymb,mathrsfs,amsfonts,dsfont}
\usepackage{epsfig,graphicx}
\usepackage{tabularx}
\usepackage{blkarray}
\usepackage{slashed}
\usepackage{caption}
\usepackage{lipsum} 
\usepackage[toc,title,titletoc,header]{appendix}
\usepackage{minitoc}
\usepackage{color}
\usepackage{multicol} 
\usepackage{multirow}
\usepackage{bm}
\usepackage{imakeidx} 
\usepackage{hyperref}  
\usepackage{etoolbox}
\usepackage{filecontents}
\usepackage{catchfile}
\usepackage{pdfpages}
\usepackage{amscd}
\usepackage{tikz-cd}
\usepackage{pifont}
\usepackage{wrapfig}

\hypersetup{
    colorlinks=true,
    citecolor=magenta,
    linkcolor=blue,
    filecolor=green,      
    urlcolor=cyan,
}
\usepackage[capitalise]{cleveref}
\usepackage{subcaption}
\usepackage{enumitem}
\usepackage{mathtools}
\usepackage{physics}
\usepackage[linesnumbered,ruled,vlined,algosection]{algorithm2e}
\SetCommentSty{textsf}
\usepackage{epigraph}
\usepackage{geometry}

\usepackage{thmtools}
\declaretheorem[numberwithin=section]{theorem}

\declaretheorem[numberwithin=section]{lemma}
\declaretheorem[numberwithin=section]{proposition}

\declaretheorem[sibling=theorem]{corollary}
\usepackage{changepage}

\makeindex[columns=3, intoc, title=Alphabetical Index, options= -s index.ist]

\usepackage{fancyhdr}
\makeatletter
\def\chaptermark#1{\markboth{\protect\hyper@linkstart{link}{\@currentHref}{Chapter \thechapter ~ #1}\protect\hyper@linkend}{}}
\def\sectionmark#1{\markright{\protect\hyper@linkstart{link}{\@currentHref}{\thesection ~ #1}\protect\hyper@linkend}}
\makeatother

\newcommand{\cmark}{\textcolor{green!60!black}{\ding{51}}}
\newcommand{\xmark}{\textcolor{red!80!black}{\ding{55}}}

%% file: macros.tex
\usetikzlibrary{positioning}

\sethlcolor{yellow}
\newwrite\commentsfile
\AtBeginDocument{%
  \IfFileExists{comments.txt}{}{%
\immediate\openout\commentsfile=comments.txt
  }%
}
\AtEndDocument{%
  \immediate\closeout\commentsfile
}

\newcommand{\readcomments}{%
  \IfFileExists{comments.txt}{%
    \newcommand{\allcomments}{\input{comments.txt}}%
  }{%
    \def\allcomments{}%
  }%
}

\makeatletter
\renewcommand\paragraph{\@startsection{paragraph}{4}{\z@}%
  {-3.25ex \@plus -1ex \@minus -.2ex}%
  {1em}%
  {\normalfont\normalsize\bfseries}}
\makeatother

\newcommand{\probP}{\text{I\kern-0.15em P}} 

\newtheoremstyle{breakthm}
  {\topsep}        
  {\topsep}        
  {\itshape}       
  {}               
  {\bfseries}      
  { }              
  {\newline}       
  {}   

\theoremstyle{breakthm}
\newtheorem{definition}{Definition}[subsection]

\newtheorem{thm}{Theorem}[section]

\newtheorem{remark}{Remark}[section]
\tcolorboxenvironment{boxeddefinition}{
  boxrule=0.5pt,
  colframe=black,
  colback=white,
  enhanced,
  rounded corners=southeast, 
  arc=3pt,               
  before skip=10pt,
  after skip=10pt
}

\numberwithin{equation}{subsection}

%% file: Parts/1_Introduction.tex
\section{Introduction}
\label{sect:introduction}
Artificial General Intelligence (AGI) encompasses a wide and heterogeneous family of agent architectures: reinforcement learning, universal AI agents, active inference, causal reinforcement learning, bounded-optimal agents, and schema-based learning architectures, among many others. Although these frameworks differ in motivation, mathematical
formulation and operational behavior, all share the same structural intuition: an agent is a compositional system that transforms perceptual inputs into actions while updating internal states according to characteristic computational laws. 
Despite this apparent unity, there is currently no formal framework that allows us to compare AGI architectures, state precise relationships between them, derive structural guarantees, or identify principled reasons for their empirical differences. Most existing theories provide results internal to their own formulation (e.g. convergence of value iteration in RL, coherence of Bayesian inference), yet they do not explain how these theories relate to one another, nor how an agent defined in one formalism may be translated, embedded, or approximated within another.

In this regard, this paper proposes a comparative framework based on Category Theory. Rather than viewing an architecture as a concrete algorithm, we treat it as a structured theory of computational interconnections: a specification of admissible interfaces, primitive components, and compositional wiring patterns. This shifts the focus from implementation details to structural organization. Crucially, we distinguish two layers that are often conflated: on one hand, the \emph{syntactic layer}, which governs how operative modules may be composed, and the \emph{knowledge management layer}, on the other hand, which governs how information is represented, transformed, and reused within that structure. Architectures may exhibit similar module flows while differing fundamentally in how they encapsulate models, aggregate evidence, or modularize experience. Thus, making this separation explicit is essential for identifying genuine structural differences and formally characterizing architectural properties.
We formalize these layers using hypergraph categories as a compositional language. They provide a natural algebraic framework to describe types, multi-input/multi-output modules, copying and merging of information, and wiring patterns in a uniform manner. This choice is structural rather than aesthetic. That is, the hypergraph structure captures the relational and resource-sensitive character of information flow in agent architectures, allowing architectural constraints to be expressed independently of any particular semantics.

That is, an architecture syntax does not determine a single representation of knowledge, but constrains how syntactic configurations admit and transform informational states. This separation allows us to analyze how different knowledge organizations inhabit the same syntactic scaffold, and how syntactic changes modify the admissible space of knowledge structures, enabling principled comparisons within and across different architectural families.

Furthermore, we introduce a third structural component, namely, the \emph{constraint layer}. While the syntax and knowledge layers specify, respectively, the admissible compositional patterns and the space of internal informational structures, they do not by themselves completely determine which concrete realizations qualify as valid instances of the architecture. The constraint layer fills this gap by encoding the restrictions that candidate implementations must satisfy in order to be considered admissible agents realizing the architecture.
This layer consists of two complementary components: a \emph{relational interface} and a \emph{constraint system}. The relational interface between syntax and knowledge is modeled through a profunctorial correspondence. That is, the architectural syntax does not prescribe a unique representation of knowledge, but instead it constrains how syntactic configurations can access, interact with, and transform informational states. This abstraction allows multiple knowledge organizations to inhabit the same syntactic scaffold, while ensuring that syntactic variations systematically induce changes in the admissible space of knowledge structures. As a result, it enables principled comparisons across architectural families at the level of information flow and interaction.
The constraint system, in turn, captures the mathematical conditions that must be satisfied by any concrete realization of the architecture. Rather than representing constraints as isolated formulas, we organize them as an indexed family over architectural scopes, formalized as a fibrational structure of admissible constraints. Intuitively, this provides a compositional and implementation-independent language for expressing heterogeneous requirements, such as representability conditions, fixed-point equations, optimality principles, or factorization properties, and for evaluating their satisfaction under semantic realizations. This formulation ensures that constraints are both structurally grounded in the architecture and robust under architectural transformations.  

Concrete agents are instantiations of a specific architecture, and are defined as monoidal functors
\[I:\mathcal{G}_A\to\mathcal{E}, \qquad J:Know_A\to\mathcal{E},\]
which interpret the abstract architecture $A$ inside a semantic universe $\mathcal{E}$ (e.g., $\mathsf{Stoch},\mathsf{FinStoch},\mathsf{Set}$, Kleisli categories ). Such functors provide the implementation-level semantics of the agent while fulfilling the constraint layer that characterize the architecture.

This construction naturally extends to a Grothendieck fibration
\[ Agents(\mathcal{E}) \;:=\;
\int_{A \in ArchAgents} Agents(A,\mathcal{E}). \] 
where the base category $\mathbf{ArchAgents}$ consists of different architecture types as well as their structure-preserving translations, and the fiber over each architecture contains all its compatible implementations. Morphisms between architectures correspond to translations of computational structure and induce reindexing functors relating their agents.

In order to study the theoretical capabilities of architectures, we also introduce for each architecture $A$ a first formalization of structural properties as posets/categories
\[\mathbf{Props}(A),\]
capturing notions such as convergence guarantees, expressivity, sample efficiency, stability, causal identifiability, or modularity. 

Furthermore, each concrete agent is equipped with a property valuation, assigning to each architectural property a degree of fulfillment based on its semantic implementation. This provides a principled mathematical basis for relating:
\begin{itemize}
    \item \emph{structural guarantees} (properties of the architecture),
    \item \emph{semantic capabilities} (properties of the agent), and
    \item \emph{performance measurements} (empirical evaluations).
\end{itemize}

In summary, this paper proposes a unified category-theoretic framework in
which:
\begin{enumerate}
    \item architectures are algebraic theories of interconnection (free hypergraph categories with designated structural diagrams and specific constraints),
    \item agents are admissible semantic interpretations of architectures (monoidal functors into a system category),
    \item properties form a functorial structure of theoretical guarantees (posets of laws derivable from the architecture),
    \item and the entire framework organizes into a fibration of agents over architectures, enabling principled comparisons, translations, and analyzes across formalisms.
\end{enumerate}

%% file: Parts/2_Related_Work.tex
\section{Related Work}
\subsection{AGI and Cognitive Architectures Related Work}

There exists very good extensive AGI functional descriptions \cite{Goertzel14}, cognitive architectures for AGI \cite{Goertzel2023,Goertzel2014P1,Goertzel2014P2,Laird2012Robotics,Laird2017,Laird2012Soar,Laird1990}
\cite{Rosenbloom2016a,Rosenbloom2016b,Sun2005,Sun2016}
some formal frameworks for AGI \cite{Catt2023,He2025,Hutter2005,Hutter2024,Veness2011,Wang2012}, 
empirical AGI benchmarking frameworks \cite{chollet2019}, and even proposals for a  Common Standard Model for Cognitive Architectures \cite{Laird2017}. 
Yet, there still is a need for a formal comparative framework that allows precise, mathematically grounded comparison among different candidate AGI architectures clearly and formally exposing their differences, dependencies and properties. 
As Laird, Labiere and Rosenbloom \cite{Laird2017} pointed out, one important attempt to bring several of these threads back together was work on a “generic architecture for humanlike cognition”  by Goertzel, Pennachin, and Geisweiller \cite{Goertzel2014P1}, which conceptually amalgamated key ideas from the CogPrime \cite{Goertzel2014P1,Goertzel2014P2},
CogAff (Sloman 2001), 
LIDA, MicroPsi (Bach 2009), and 
4D/RCS architectures, plus a form of deep learning (Arel, Rose, and Coop 2009). 
A number of the goals of that effort were similar to those identified for the standard model of the mind \cite{Laird2017}; however, the result in \cite{Goertzel2014P1} was more about assembling disparate pieces from across these architectures rather than identifying what is common among them, with a bias, thus, more toward completeness than concord. Whereas, the emphasis of the standard model of the mind is more towards concord and convergence. 
Regarding our view on AGI, we also borrow ideas from the definition of Artificial Intelligence based on Wang´s \cite{Wang2019} and its different types based on Legg and Huttter´s definitions \cite{Legg2007}.
We also ascribe that  Life-long continual cumulative learning is essential for AGI as pointed out by Thorisson and colleagues  \cite{Thorisson2019} and that intelligence is measured by the efficiency of skill-acquisition on unknown tasks as described in Chollet´s \cite{chollet2019} "Abstract and Reasoning Corpus"  (ARC-AGI) benchmark to measure fluid intelligence.  

\subsection{Category-Theoretic Related Work}

Due to space limitations, we only highlight the main research areas that intersect with the framework proposed in this paper. Further details are included in the extended version \cite{RiscosExtended}. 
We have been initially motivated to provide an updated view of  "Machines in a Category" \cite{arbibfuzzy,MachinesCategoryExpository,ARBIB1975313,ARBIB19809} towards AGI Architectures in a Category.
Current work on Machine Learning from a Category-theoretic (CT) perspective is reviewed by \cite{Gavranovic2024,Jia2025,Schiebler2021}, following the  seminal work by Fong and collaborators \cite{Fong2019}.
More specifically, work on Reinforcement learning from a CT perspective is very relevant for our work. In this regard, seminal current research by Hedges and Rodriguez-Sakamoto \cite{Hedges2023,Hedges2024ReinforcementLI} provides an important step towards formalizing RL under CT.
Bakirtzis and colleagues \cite{BakirtzisJMLR2025,BakirtzisECAI} also provide a very interesting different view on reusability, reducibility and compositionality in reinforcement learning. Mahadevan´s seminal work defines \cite{MahadevanUDM,MahadevanURL}
Univeral Reinforcement Learning, and by specialization RL, as essentially a process of stochastic metric coinduction in a universal
coalgebra defined over a topos of (action)value functions, where all TD-type algorithms can be analyzed in terms of iterative procedures for discovering final coalgebras.
Additionally, a few research papers have begun to analyze certain aspects of AGI from a CT point of view.
The book by  Hedges and colleagues is a very original and important first attempt to bring the Categorical perspective and formalization to the General Intelligence endeavor \cite{SwanGI2022}. They provide a  review of the wide IA capabilities spectrum and also propose Categorical Cybernetics \cite{Capucci2022} as the core formal foundation. 
On the other hand, some recent papers also provide with other views on different uses of CT towards AGI \cite{AbbottCategoryTheoryAGI,Mahadevan26,sennesh2023computingcategoriesmachinelearning}
\cite{YanAGICatLogic,yuan2023categoricalframeworkgeneralintelligence}.
While they all represent valid attempts to some specific AGI capabilities, they lack overall generality. 
We take a different approach by emphasizing the architectural level with different structural, knowledge and semantic layers allowing a general comparative framework that is able to encompass many possible candidate AGI architectures. 
Research on Systems Theory \cite{Baez2015categoriescontrol,Bonchi2019,Coecke2016,Fong2016,Libkind2025,Myers21,Myers23} and Causality and Markov processes from the CT perspective \cite{Baez_2016,DiLavore2023,DiLavore2025,Fritz2020,jacobs2019causalinference,Kissinger2017,Mahadevan23,Mahadevan25,Perrone2024} also relates to our semantic implementation of specific agents as further detailed in \cite{RiscosCATSBL,RiscosExtended}.

%% file: Parts/3_ArchAgents_Category.tex
\section{The ArchAgents Category}
\label{sect:ArchAgentCat}
In this section we introduce the category $\mathbf{ArchAgents}$, whose objects are abstract agent architectures and whose morphisms are translations between them. Our guiding principle is that an architecture does not specify algorithms, models, or learning procedures, but rather defines algebraic constraints on admissible components, their interconnections, and specific knowledge management mechanisms. Concrete agents arise only later as semantic interpretations of these architectures.

We adopt the formalism of hypergraph categories \cite{fong2016algebraopeninterconnectedsystems,fong2019hypergraphcategories} as the foundation for representing network-style architectures for syntax and knowledge dynamics.
Hypergraph categories are symmetric monoidal category equipped with special commutative Frobenius structures, enabling graphical representation of wiring diagrams, feedback, and open interconnections.

\subsection{What is an agent´s  architecture?}
Before introducing the formal definitions, we clarify what we mean by agent´s architecture and provide some high-level examples.
We define an agent´s architecture as its blueprint that defines some structural and domain constraints. In that regard, we first distinguish two orthogonal architectural dimensions:
\begin{enumerate}
    \item Syntactic structure:  it determines how the map from perception to action is organized as a compositional flow. This includes the intermediate stages involved in decision-making, learning, or model use. The architecture specifies which modules exist, how they are connected, and what their inputs and outputs are, but it does not specify the internal algorithms or implementations of these modules.
    \item Knowledge structure: it determines what  kinds of representational knowledge structures (what we call knowledge units)  may exist or be constructed, and what structural transformations can be performed over them. That is, the architecture determines the different ways in which information from the environment can be encapsulated, as well as the transformations the agent is allowed to perform on such knowledge carriers.
\end{enumerate}
Importantly, the architecture does not determine:
\begin{itemize}
    \item the specific models used to instantiate those knowledge units,
    \item nor the concrete algorithms used for implementation.
\end{itemize}
These two architectural dimensions, although orthogonal, are obviously coupled inside the architecture. That is, the operational modules induce which transformations over knowledge are used, and the representational structure of knowledge must be reflected in the syntactic interfaces (ports) of the architecture.
We now illustrate this notion of architecture and its two dimensions through some examples.

\subsubsection{Illustrative Example: A Ford Assembly Line}

Before introducing the formal definitions and more complex examples, it is useful to distinguish informally between \emph{syntax} and \emph{knowledge} through a simple industrial example.
Consider a car assembly line, such as a classical Ford-style manufacturing process. At an abstract level, the assembly line imposes a structured workflow: the chassis must be prepared before the engine can be mounted, the engine must be installed before certain electrical components are connected, and the wheels cannot be attached before the axle and suspension are in place. In other words, there is a constrained order of admissible operations. This defines the \emph{syntactic structure} of the process.

More precisely, the syntax specifies:
\begin{itemize}
    \item which stages or modules exist in the assembly pipeline,
    \item which outputs of one stage may serve as inputs to another,
    \item and which compositions are admissible or forbidden.
\end{itemize}

For example, ``install wheel'' cannot be composed meaningfully before ``mount axle'', and ``tighten wheel bolts'' only makes sense once the wheel has already been positioned. These are not implementation details, but structural constraints on the organization of the process.
However, once this syntactic organization is fixed, each stage may still admit multiple concrete realizations. For instance, the task ``tighten wheel bolts'' may be carried out:
\begin{itemize}
    \item manually by a human worker tightening the bolts in a specific order,
    \item or by an industrial robot tightening all the bolts at the same time.
\end{itemize}

Likewise, ``paint chassis'' may be realized by different methodologies, and ``inspect alignment'' may rely on some human judgments or automated rule systems. These alternatives do not change the compositional structure of the assembly line itself; they only change \emph{how} the processes are done. This corresponds to the \emph{knowledge layer}.
Thus, in this example, the distinction is the following:

\begin{quote}
\textbf{Syntax} determines the admissible workflow of the architecture, that is, the organizational pattern of the process.\\
\textbf{Knowledge} determines the admissible internal realizations of each component determining how to proceed, that is, the methods, models, or mechanisms used to execute each stage.
\end{quote}

This distinction is central for agent architectures. Two agents may share exactly the same syntactic organization, for example, the same perception $\to$ inference $\to$ action $\to$ update pipeline, while differing radically in the internal knowledge structures or procedures used inside each module. Conversely, two agents may use similar knowledge mechanisms but differ in how those mechanisms are compositionally organized. Our framework treats these as two orthogonal but coupled architectural dimensions.

\subsubsection{Example 2: The Doctor's clinic}
Consider a doctor in a daily clinical practice. Each patient activates the same operational scheme: the doctor gathers symptoms, formulates a provisional diagnosis, prescribes a treatment, observes the outcome, and revises the decision if necessary. The flow $observe \to decide \to act \to evaluate$ remains stable over time, regardless of the specific diseases encountered. This constitutes the operational architecture and what evolves is the internal organization of knowledge.

When confronted with a new disease that shares symptoms with a known previous one, the doctor may initially be forced to split what was previously considered a single condition into two distinct explanatory models. Once differentiated, each one can be refined by incorporating new distinguishing features. As more cases accumulate, broader regularities may emerge, for example, that certain age groups respond worse to specific symptoms, leading to more general rules. Eventually, such regularities may be codified into systematic protocols guiding future decisions.
Throughout this process, the operational structure remains unchanged, but the organization of knowledge becomes progressively refined. That is, differentiated, structured and generalized.

Now consider a second doctor with exactly the same operational flow. Externally, both doctors interact with patients in identical ways. The difference lies in their evaluation step, how they manage knowledge.
The second doctor has constraints on the structure of his understanding. He assumes a single unified explanatory model for all diseases, so he is constrained to do parametric updates within a fixed hypothesis structure. When faced with unexpected outcomes, he merely adjust confidence levels in the treatments. If a therapy fails more often than expected, he updates the internal probabilities (e.g., from 80\% to 60\% effectiveness), but he does not consider the possibility that two distinct diseases may underlie similar symptoms.

When a genuinely distinct disease appears, the first physician can isolate it conceptually and construct a new explanatory model. The second doctor, lacking structural operations such as model differentiation, is forced to fit all cases into a single schema. His performance degrades, not because of differences in perception or action, but because his architecture permits only parametric uncertainty adjustment, not structural reorganization of knowledge.
The difference lies not in observation or action, but in the admissible operations over knowledge units.

\subsubsection{Example 3: The Navigation Robot}
Consider two robots that must navigate a city daily to reach a café. Both update their knowledge using the same learning rule, That is, when a route takes longer than expected, they adjust their internal time estimates. Thus, in terms of parameter updating, they are identical.
Yet, the first robot processes its perceptions monolithically. Each time it observes traffic, weather, and time of day, it treats this information as a single undifferentiated block, computes a globally optimal action, executes it, and then updates its model. Its flow is: $global\;perception \to global\;decision \to action \to global\;update$
On the other hand, the second robot introduces structural decomposition. Upon receiving perceptual input, it first separates information into distinct components: traffic state, weather conditions, etc. It then uses its model to generate predictions separately (e.g., travel time estimation, risk assessment), and only then combines these partial outputs to determine the next action. After observing the outcome, it updates its knowledge only taking into account the component responsible for the error.

However, in this case the difference lies entirely in the syntactic organization of the perception–action pipeline, not in the way knowledge itself is modified.
Although both robots use the same rule for adjusting estimates, the organization of their perception–action pipeline differs. The second robot can isolate subtasks and errors, reuse components, and adapt locally since it does modularize its perception. On the other hand, the first robot must always use and update its entire model globally.
Thus, the distinction is not in how knowledge is updated, but in how the perception–action flow is structured.

\subsubsection{The Constraint Layer}

So far, we have described an architecture in terms of two interacting layers: a syntactic layer, specifying admissible compositional patterns, and a knowledge layer, specifying admissible representations and transformations of information. However, this description is still too permissive. In general, not every syntactically well-formed composition nor every formally admissible knowledge construction should be considered a valid instance of a given architecture.

The role of the \emph{constraint layer} is precisely to restrict this space of possibilities.

At an abstract level, the constraint layer specifies \emph{laws} that admissible architectures must satisfy. These laws do not introduce new components or transformations, but rather impose conditions on how existing ones may be used, combined, or interpreted. In this sense, the constraint layer refines the raw generative capacity of the syntactic and knowledge layers into a structured class of admissible systems.

More concretely, constraints may act on different aspects of the architecture:

\begin{itemize}
    \item \textbf{Syntactic constraints:} restricting admissible implementations over the wiring patterns.
    
    \item \textbf{Knowledge constraints:} restricting the realization of the knowledge units or transformations. For example, requiring that beliefs form a Bayesian model, that updates are likelihood-based, or that representations admit only parametric updates.
    
    \item \textbf{Interface constraints:} enforcing compatibility conditions between syntax and knowledge, ensuring that the types and arities of syntactic modules align with the structure of knowledge representations.
    
    \item \textbf{Dynamical or consistency constraints:} imposing global properties such as convergence, coherence, or invariance under certain transformations.
\end{itemize}

From a category-theoretic perspective, these constraints can be understood as objects of \emph{subcategories}, or more generally as \emph{properties over designated fragments of the architecture}. That is, while the architecture defines a free or generative compositional structure, the constraint layer carves out a subtheory of admissible realization when this compositional structure moves from theory to practice.

This distinction is essential. Two architectures may share exactly the same syntactic generators and knowledge primitives, yet differ in their constraint layer. In such cases, the difference does not lie in what can be expressed, but in what is considered \emph{valid} during the implementation. For example, a reinforcement learning architecture and a fully Bayesian agent may share similar structural components, yet differ in whether belief updates must satisfy Bayesian coherence, or whether value functions must satisfy Bellman consistency.

In summary, the constraint layer transforms an architecture from a purely generative specification into a \emph{normative theory}: it determines not only what can be built, but what counts as an admissible realization. This layer will play a central role in distinguishing agent families, defining admissibility of agents, and enabling precise comparisons between architectures.

\subsection{Architectural Presentations and Generated Categories}
First, we explain the methodological framework used to define the specific hypergraph categories considered in this work. Our construction relies on the relationship between colored PROPs and hypergraph categories. For this reason, we introduce a notion of hypergraph presentation that mirrors the presentation of a free colored PROP.
The guiding principle is the classical construction of a free colored PROP from a signature. That is, one specifies a set of colors (types), a collection of generating morphisms with prescribed input and output profiles over those colors, and a family of equations between composite expressions. In the present setting, this syntactic data is further enriched by freely adjoining, on each color, the structure of a special commutative Frobenius algebra, subject to the usual Frobenius, unit, counit, associativity, commutativity, and specialness axioms. Some researchers have worked before on the correspondence between such presentations and hypergraph categories. That is, every hypergraph category arises (up to equivalence) from a PROP equipped with compatible Frobenius structures on its objects, and conversely, any PROP presented in this way canonically induces a hypergraph category. \cite{Zanasi_2017,fong2018sevensketchescompositionalityinvitation}

Accordingly, we present free hypergraph categories directly in terms of generators, relations, and Frobenius structure. This is not an ad hoc reformulation, but rather the natural categorical analog of the classical presentation of free PROPs, internalized to the hypergraph setting. We adopt the notation $\mathrm{Types}$ and $\mathrm{Gen}$ for the underlying signature data, instead of other more traditional symbols such as $\Sigma_0$ and $\Sigma_1$, in order to emphasize the semantic interpretation that will be relevant in later sections.

\begin{definition}[Hypergraph Presentation]
A \emph{free hypergraph presentation} is a triple $(Types, Gen, Eq)$ where:
\begin{itemize}
    \item $Types_A$ is a set of formal object symbols representing the interfaces of the syntactic or knowledge part, such as perceptual channels, action interfaces, or memory ports,

    \item $Gen_A$ is a set of morphism symbols representing primitive components. Each generator $g \in Gen_A$ is equipped with a typing
    \[g : X_g \to Y_g,\]
    where $X_g$ and $Y_g$ are formal tensor expressions over $Types_A$ under a symmetric monoidal product $\otimes$.
    Generators specify admissible interconnections between interfaces (e.g. policy modules, inference modules, causal structural learners, perceptual module, memory updates, etc.), without any algorithmic or semantic interpretation,
    
    \item  $Eq_A$ is a (possibly empty) set of equations consisting solely of:
    \begin{itemize}
      \item the axioms of symmetric monoidal categories,
      \item the Frobenius algebra axioms induced by the hypergraph structure,
      \item additional purely syntactic wiring equalities required by the architectural presentation.
    \end{itemize}

\end{itemize}

This presentation generates the following free hypergraph category:
\[\mathsf{Hyp}\langle Types, Gen \mid Eq \rangle .\]
\end{definition}

\begin{remark}
We assume that the Frobenius structure on types is part of the ambient hypergraph setting and therefore need not be explicitly specified in $Eq$. 
\end{remark}

From this definition, we can present the two main ingredients that define an architecture: the syntactic and the knowledge dimensions. We formalize both using hypergraph categories because their symmetric monoidal structure, together with special commutative Frobenius algebras, provides a canonical way to represent compositional wiring diagrams with copying, merging, and feedback. This choice is independent of implementation details and ensures that architectures can be compared at the structural level. Moreover, working within this algebraic setting equips us with a well-developed categorical toolkit for reasoning about equivalence, compositionality, and invariants. Thus, enabling the formal comparison and prove of structural properties for different AI architectures.

\paragraph{Syntactic layer.}
A \emph{syntactic presentation} $A := (STypes_A, SGen_A, SEq_A)$ is a
hypergraph presentation whose objects represent syntactic interfaces
(e.g.\ perceptual channels, action interfaces, memory ports) and whose
generators represent primitive syntactic components. It generates the
\textbf{syntactic layer category}
\[\mathsf{Syn}_A := \mathsf{Hyp}\langle STypes_A, SGen_A \mid SEq_A \rangle .\]

\paragraph{Knowledge layer.}
The syntactic layer specifies admissible wiring patterns between
interface types, but it does not by itself determine how persistent internal knowledge
is structured, transformed, or accessed. These aspects are syntactic by
nature and are specified independently of any concrete learning algorithm or
semantic interpretation.
Analogously, a \emph{knowledge presentation}
$(\mathsf{KTypes}_A,\mathsf{KGen}_A,\mathsf{KEq}_A)$ is a hypergraph presentation
whose objects represent types of internal knowledge resources and whose
generators represent admissible knowledge transformations, such as updating,
combining, encapsulating, or discarding knowledge. It induces the
\textbf{knowledge layer category}
\[
\mathsf{Know}_A := \mathsf{Hyp}\langle \mathsf{KTypes}_A, \mathsf{KGen}_A
\mid \mathsf{KEq}_A \rangle .
\]

\subsection{Syntax Patterns and Admissible Workflows}
After the definition of the syntactic and knowledge hypergraphs, we are going to identify what are the possible morphisms inside each category. This is important since the syntactic morphisms represents all the possible ways that an agent could possibly behave, but under each architecture, the implemented agents will only follow one of those posibilities, these is what we define as the syntactic diagram. It is also important to interpret what the knowledge layer morphisms will be, since will determine all the possible workflows of transformations over the knowledge units. This is what we call knowledge workflows, that will be instantiated during implementation, and for example it is what can truly differentiates the learning behaviour of two agents during empirical evaluation, going further than the specific properties of the algorithms used.

\begin{definition}[Syntax Diagram]
Given $\mathsf{Syn}_A$, a \emph{syntax diagram} $\mathcal{G}_A$ is the full
symmetric monoidal subcategory of $\mathsf{Syn}_A$ generated by a distinguished architectural pattern, that is, a morphism $g_A \in \mathsf{Syn}_A$, together with the class $\langle g_A \rangle \subseteq \mathrm{Mor}(\mathsf{Syn}_A)$ freely generated from $g_A$ under:
\begin{itemize}
    \item symmetric monoidal structural isomorphisms,
    \item Frobenius algebra equations,
    \item composition and tensoring with identities.
\end{itemize}
\end{definition}

The category $\mathcal{G}_A$ specifies the admissible \emph{syntax diagrams}, i.e. the syntactic compositions considered meaningful at the syntactic level. The generating morphism $g_A$ acts as an syntactic skeleton, without fixing any semantic interpretation. Analogously, we identify the morphisms of an architectural knowledge category as knowledge workflows. \\

\begin{definition}[Knowledge Workflow]
A knowledge workflow is a morphism in the knowledge layer category $\mathsf{Know}_A$, that is,
\[ w : X \to Y \in \mathsf{Know}_A,\]
\end{definition}

Knowledge workflows are not further restricted at the architectural level, as they are designed an made explicit during the implementation and the specific algorithmic instantiations.

\subsection{The Constraints layer}
The syntax layer determines the permitted wiring patterns between interfaces, while the knowledge layer defines the knowledge units and the permitted ways to perform operations on them. However, these two layers do not by themselves specify the constraints that the architecture imposes on potential implementations, meaning what can be considered admissible agents within an architecture. 
This is what we denote as the constraints layer of the architecture and it covers the following two aspects: 
\begin{enumerate}
    \item the relationship between the syntactic and the knowledge layers, i.e. which syntactic and knowledge types are related and which syntactic generators induce knowledge workflows in the implemented agents, and
    \item the different conditions or constraints that the architecture imposes over the admissible realizations, e.g. Bellman consistency of update operators in RL, expectimax-conducted policy in UAI architectures or the obligation of some knowledge units to be causal in the CRL architecture.     
\end{enumerate} 

\subsubsection{The Relational Interface}
The syntax layer determines admissible wiring patterns between interfaces, but it does not specify how it is related with the knowledge layer that formalizes the knowledge management dynamics. This interaction is captured by a profunctorial interface.

\begin{definition}[Relational Interface]
Let $\mathcal{G}_A$ be an syntax diagram and $\mathsf{Know}_A$ be the knowledge layer category.
We define the \emph{relational interface} $\Phi_A$ that relates both categories as
\[ \Phi_A : \mathcal{G}_A \nrightarrow \mathsf{Know}_A,
\qquad
\Phi_A : \mathcal{G}_A^{\mathrm{op}} \times \mathsf{Know}_A
\longrightarrow \mathbf{Set}.
\]
\end{definition}

The profunctor $\Phi_A:\mathcal{G}_A \nrightarrow Know_A$ specifies where and how syntactic components may interact with internal knowledge resources, without committing to any particular representation or transformation mechanism for knowledge.

A profunctor is appropriate here because the relationship between syntax and knowledge is not functorial. That is, a syntactic component may access multiple knowledge types, and a knowledge type may serve multiple syntactic contexts. Thus, the interaction is relational rather than strictly functorial.
Moreover, the orientation $\Phi_A:\mathcal{G}_A \nrightarrow Know_A$ captures an essential asymmetry in this relation. That is, the syntactic components are defined independently, while their admissible interactions may depend on the available knowledge structures. In this sense, syntax is parametrically constrained by knowledge, but it is not determined by it.

Although we formalize the syntax–knowledge interaction as a profunctor, alternative categorical interfaces could be considered. In particular, this interaction may admit a refinement in terms of optic-like structures (e.g. lenses or more general profunctor optics), which explicitly model bidirectional access, update, and contextual embedding of knowledge inside syntactic components. Conversely, one may also consider simplifying the interface by restricting the profunctor to be $\mathbf{Bool}$-valued, capturing only the existence of a relation rather than its full set-valued structure, thereby avoiding potential overformalization. However, for the present work, the profunctorial formulation provides a sufficiently general and implementation independent abstraction. It captures dependency and admissibility without imposing additional algebraic structure. A systematic optic-based treatment is deferred to future work.

\begin{remark}[Object-Level Interpretation]
For $s \in \mathrm{Ob}(\mathcal{G}_A)$ and $k \in \mathrm{Ob}(\mathsf{Know}_A)$,
the set $\Phi_A(s,k)$ indicate if there exist a relation between the syntactic type $s$ and the knowledge type $k$. That is, the profunctor induces a structural partition among the syntactic types: the types that interact with some type of knowledge $k$ ($\Phi(d,k) \neq \varnothing $) or not ( $\Phi(d,k) = \varnothing $)
\end{remark}

\begin{remark}[Functorial Action]
Functoriality of $\Phi_A$ ensures that syntactic refinement and knowledge transformation act coherently on admissible interactions, contravariantly in syntactic diagrams and covariantly in knowledge transformations.
\end{remark}
Intuitively, the profunctor does not prescribe how knowledge is represented or transformed, but only specifies where and how architectural structure may interact with internal knowledge resources.

\begin{remark}[Generator-Level Classification]
When the syntactic interface is specified by its action on the primitive generators of $\mathcal{G}_A$, it induces a natural classification of syntactic components into knowledge-using, knowledge-transforming or knowledge-agnostic elements.
\end{remark}

\subsubsection{The Constraint System}
We define the Constraint System as an indexed family of constraints organized fibrationally over the different architectural scopes. This is intended to capture those statements that distinguish architectures not only by their compositional organization, but also by the mathematical conditions imposed on their admissible realizations. First we introduce the category of scopes given the syntactic and knowledge layers and the relational interface.

\begin{definition}[Category of Scopes]
Let $\mathcal G_A$ be a syntax diagram, $\mathsf{Know}_A$ a knowledge layer, and $\Phi_A$ a relational interface.
The category of scopes $\mathcal C_A$ is a category whose objects are designated architectural fragments, including:
\begin{itemize}
    \item types, generators or morphisms in $\mathcal G_A$ or $\mathsf{Know}_A$,
    \item interface fragments induced by $\Phi_A$,
    \item and structured combinations of the above.
\end{itemize}
Morphisms in $\mathcal C_A$ represent inclusions, refinements,
or transformations of scopes.
\end{definition}

\begin{definition}[Architectural Constraint Fibration and Constraint]
Given a syntax diagram $\mathcal G_A$, a knowledge layer $\mathsf{Know}_A$, and relational interface $\Phi_A$, an architectural constraint fibration is a fibration
\[
\pi_A : \mathcal P_A \longrightarrow \mathcal C_A,
\]
where for each scope $X \in \mathcal C_A$, the fiber
$\mathcal P_A(X)$ is a poset (or category) of constraints
that can be imposed on realizations of $X$.

A \emph{constraint} is a pair $\rho = (X,\, P)$, with $X \in \mathcal C_A $ and $ P \in \mathcal P_A(X)$. We will usually denote $X$ as $\mathrm{scope}(\rho)$.
\end{definition}

\begin{remark}[Interpretation of the Constraint Fibration]
Intuitively, the fiber $\mathcal P_A(X)$ collects constraints that can
be imposed on semantic realizations of the architectural fragment $X$.
In typical cases, such constraints can be understood as specifying
admissible subclasses of a suitable space of realizations,
$\mathrm{Mod}_{\mathcal E}(X)$. $\mathcal{E}$ is how we will denote the category of implementations or realization when formalizing agents in the next section \ref{sect:AgentCat}. 

We deliberately leave the precise construction of
$\mathcal P_A$ abstract, as different semantic settings may
instantiate it in different ways (e.g. subobject lattices,
logical predicates, or structured constraint classes).
\end{remark}

\begin{definition}[Constraint Layer]
The \emph{constraint layer} is given by
\[\mathcal T_A = (\Phi_A,\; \pi_A : \mathcal P_A \to \mathcal C_A,\; \mathcal R_A),\]
where $\mathcal R_A =\{\rho_i=(X_i,P_i) \; s.t. \;X_i \in C_A \wedge  P_i \in P_A(X_i)\}$ is a distinguished set of constraints.
\end{definition}

\begin{remark}[Constraints Typology]
\label{remark:Constraint typology}

Although constraints are formalized as constraints in the fibers
of the fibration $\pi_A : \mathcal P_A \to \mathcal C_A$, it is useful
to retain a coarse typology reflecting the possible different mathematical
nature of such constraints.
Concretely, properties in $\mathcal P_A(X)$ typically arise from
recurring structural patterns, including:

\begin{itemize}
    \item \textbf{Diagrammatic / commutativity constraints:}
    requiring that certain diagrams in a realization commute;
    
    \item \textbf{Factorization constraints:}
    requiring that a realization factors through a designated
    decomposition or interface;

    \item \textbf{Membership / representability constraints:}
    restricting realizations to lie within a specified class
    (e.g.\ value-like objects, probabilistic models);

    \item \textbf{Fixed-point constraints:}
    requiring the existence of elements satisfying
    $x = F(x)$ for a designated operator;

    \item \textbf{Optimality / extremality constraints:}
    imposing that a realization is optimal with respect to
    a given order, functional, or preference structure;

    \item \textbf{Order or inequality constraints:}
    expressing monotonicity, bounds, or preference relations;

    \item \textbf{Interface compatibility constraints:}
    relating realizations across the syntax--knowledge interface
    via $\Phi_A$.
\end{itemize}

This classification is not part of the formal definition of the
constraint system and it is not a final typology or taxonomy, but provides a useful organizational layer for understanding the types of constraints that appear in practice. Different architectures typically select constraints from several of these classes.
\end{remark}

\subsection{The Category $\mathbf{ArchAgents}$}

\begin{definition}[Agent Architecture]
An \emph{agent architecture} is a triple 
\[A := (\mathcal{G}_A,\mathsf{Know}_A,\mathcal{T}_A),\]

where 
\begin{itemize}
    \item $\mathcal{G}_A \subseteq \mathsf{Syn}_A$ is a distinguished symmetric monoidal subcategory of the hypergraph category $Syn_A$, specifying the admissible \emph{syntactic diagrams}, that is, the syntactic compositions considered meaningful at the architectural level.
    
    \item $\mathsf{Know}_A$ is a hypergraph category, called the \emph{knowledge layer category}, whose objects represent types of knowledge units/resources and whose morphisms represent admissible knowledge transformations workflows.

    \item $\mathcal{T}_A$ is the constraint layer containing the relational interface profunctor, specifying how the syntactic level is related with the knowledge level, and the constraint system containing the architectural Constraint fibration the distinguished set of constraints.
\end{itemize}

\end{definition}

The architecture $A$ constrains both the space of admissible syntactic compositions and the internal structure and transformation of knowledge, independently of any particular agent implementation or learning algorithm, whenever it fulfills the requirements of the constraint layer.

\begin{definition}[Architecture Morphisms]
Let $A,B$ be agent architectures. A morphism $F : A \to B$
consists of the following data: 

\begin{enumerate}
    \item \textbf{Structural translation:}
    a pair of symmetric monoidal functors
    \[
    F_{\mathcal{G}} : \mathcal{G}_A \to \mathcal{G}_B,
    \qquad
    F_{\mathsf{Know}} : \mathsf{Know}_A \to \mathsf{Know}_B;
    \]

    \item \textbf{Interface compatibility:}
    a natural transformation
    \[
    \Phi_A \;\Rightarrow\; \Phi_B \circ
    (F_{\mathcal{G}}^{\mathrm{op}} \times F_{\mathsf{Know}});
    \]

    \item \textbf{Constraint part translation:}
    a functor
    \[
    F_{\mathcal C} : \mathcal C_A \to \mathcal C_B
    \]
    mapping architectural scopes in $A$ to scopes in $B$,
    induced by $(F_{\mathcal G},F_{\mathsf{Know}})$ and the
    interface transformation. Together with the transportation of the properties via a functor $F^\sharp : \mathcal P_A \to \mathcal P_B$
    such that
    \[
    \pi_B \circ F^\sharp = F_{\mathcal C} \circ \pi_A,
    \]
    i.e., $F^\sharp$ maps properties over a scope $X$ in $A$
    to properties over the translated scope $F_{\mathcal C}(X)$ in $B$;

\end{enumerate}

No additional coherence conditions are imposed between the structural and constraint components, allowing abstraction-preserving architectural refinements. Notice that no kind of constraint preservation is imposed, so there exist morphisms between architectures with different levels of restrictions.

\end{definition}

\begin{definition}[The Category $\mathbf{ArchAgents}$]
The category $\mathbf{ArchAgents}$ has:
\begin{itemize}
    \item Objects: agent architectures $(\mathcal{G}_A,\mathsf{Know}_A,\mathcal{T}_A)$;
    \item Morphisms: architecture morphisms;
    
    \item Composition: given morphisms $F$ and $G$,
    their composite $G \circ F$ is defined componentwise as follows:
    \begin{itemize}
        \item \textbf{on syntax}: $(G \circ F)_{\mathcal G} :=G_{\mathcal G} \circ F_{\mathcal G}$
        \item \textbf{on knowledge}: $(G\circ F)_{\mathsf{Know}} := G_{\mathsf{Know}} \circ F_{\mathsf{Know}}$,
        \item \textbf{on the relational interface}: the natural transformation is given by the composite obtained by pasting the corresponding natural transformations,
        \item \textbf{on scopes}: $(G \circ F)_{\mathcal C}:= G_{\mathcal C} \circ F_{\mathcal C}$ and  
        \item \textbf{on constraints}$: (G \circ F)^\sharp:=G^\sharp \circ F^\sharp.$
    \end{itemize}    
    
    \item Identities:
    for each architecture $A$, the identity morphism is given by the identity functors on $\mathcal G_A$ and $\mathsf{Know}_A$, the identity natural transformation on $\Phi_A$ and the identity functor on $\mathcal C_A$ and $\mathcal P_A$.
\end{itemize}

These operations satisfy the associativity and identity laws by
functoriality of composition and associativity of natural
transformation pasting.
\end{definition}

Thus $\mathbf{ArchAgents}$ forms a category whose objects represent abstract agent architectures and whose morphisms capture structure-preserving
translations between them.

This categorical setting enables the study of general results about agent architectures, including:

\begin{itemize}
    \item \textbf{Architectural equivalence theorems}: characterizing when two
    architectures are equivalent up to symmetric monoidal and hypergraph
    structure, despite differing in their concrete decomposition into modules.

    \item \textbf{Architectural reduction theorems}: showing that specialized
    architectures (e.g., Reinforcement Learning) arise as subarchitectures,
    forgetful images, or functorial reductions of more expressive ones (e.g.,
    Causal or Structural Learning architectures).

    \item \textbf{Structural knowledge transport}: analyzing how morphisms
    between architectures induce translations between their internal knowledge
    structures, preserving or collapsing classes of knowledge transformations.

    \item \textbf{Expressive capacity and irreversibility}: relating the
    existence of non-invertible architecture morphisms to losses of internal
    representational or transformational capacity, providing a formal notion
    of architectural expressiveness. (AGI and not AGI)

    \item \textbf{Architectures of maximal expressiveness and universal families}:
    characterizing architectures that are maximal with respect to structural and
    knowledge expressiveness within a given class, and identifying families of
    architectures from which broad classes of agent designs can be obtained via
    systematic reduction or specialization. (AGI as a property of belonging to a universal family with enough expressivity/properties)
    
    \item \textbf{Generative architectural templates}: identifying minimal or
    canonical architectural patterns from which broad families of agent
    architectures can be constructed via systematic enrichment.
\end{itemize}

%% file: Parts/4_Agents_Category.tex
\section{The \textbf{Agents} Category}
\label{sect:AgentCat}
Architectures specify the abstract compositional structure of agents. Concrete agents arise as semantic realizations of these structures. In this section, we briefly formalize the notion of an agent implementing an architecture.

\subsection{Semantic Interpretations of Architectures}

We assume a symmetric monoidal category $\mathcal{E}$ representing concrete implementation systems. This category provides the semantic setting in which the architectural design is realized.
Typical examples include categories such as $\mathsf{Stoch}$,
$\mathsf{FinStoch}$, $\mathsf{Set}$, Kleisli categories, or
categories induced by programming languages.

\begin{definition}[Knowledge-Compatible Generators]
Let $A = (\mathcal{G}_A,\mathsf{Know}_A,\Phi_A,\mathcal{T}_A)$
be an architecture.
A generator $g : X \to Y$ in $\mathcal{G}_A$ is \emph{knowledge-compatible} if there exist $K_X,K_Y \in \mathrm{Ob}(\mathsf{Know}_A)$ such that
\[ \Phi_A(X,K_X) \neq \varnothing,
\qquad \Phi_A(Y,K_Y) \neq \varnothing.\]

We denote by $Gen_{\Phi}$ the class of such generators.
\end{definition}

Generators in $Gen_{\Phi}$ are syntactic components whose interfaces are related, via $\Phi_A$, to types in the
knowledge category $Know$. They belong to the syntactic category $\mathcal{G}_A$ and should not be confused with morphisms of $Know_A$. Intuitively, they represent syntactic operations whose interfaces interact with knowledge structures. Since these operations manage knowledge, they must correspond to the knowledge transformations defined in $Know_A$. We introduce this notion because an agent consists of a concrete implementation of the syntactic layer together with a realization of the knowledge layer, compatible on knowledge-relevant generators.

\begin{definition}[Agent / Admissible Realization]
\label{def:agent_final}
Let $A$ be an architecture and $\mathcal{E}$ a symmetric monoidal category.
An \emph{agent} (or admissible realization) of $A$ in $\mathcal{E}$ consists of the strong symmetric monoidal functors
\[
I:\mathcal{G}_A \to \mathcal{E},
\qquad
J:\mathsf{Know}_A \to \mathcal{E},
\]
such that:

\begin{enumerate}
    \item \textbf{(Interface compatibility)}  
    for every $g : X \to Y$ in $Gen_{\Phi}$, there exists
    $k_g : K_X \to K_Y$ in $\mathsf{Know}_A$ such that
    $I(g)$ is induced by $J(k_g)$ in $\mathcal{E}$, in a way
    compatible with the interpretation of $X$ and $Y$. Thus, the implementation of syntactic components interacting with knowledge must arise from the corresponding knowledge-level (possibly compound) transformations.

    \item \textbf{(Constraint satisfaction)}  
    for every $\rho \in \mathcal R_A$, $(I,J)$ must satisfy the constraints $(I,J)\models \rho$. (see definition \ref{def:Constraint Satisfaction})
\end{enumerate}
\end{definition}

\begin{remark}
The previous condition expresses that architectural components whose interfaces interact with knowledge types must be implemented through the corresponding knowledge transformations.
Specifically, if a generator
\[ g : X \to Y \]
belongs to $Gen_{\mathrm{K}}$, then the architecture specifies that its input and output interfaces are connected to knowledge objects. The implementation $I(g)$ is therefore required to arise from a knowledge-level transformation
\[ k_g : K_X \to K_Y \]
via the realization functor
\[ J : Know \to \mathcal{E}. \]

In this way, the behaviour of knowledge-relevant architectural
components is constrained by the structure of the knowledge category.
\end{remark}

\begin{definition}[Constraint Satisfaction]
\label{def:Constraint Satisfaction}
Let $(I,J)$ be a realization in $\mathcal{E}$. For each scope $X \in \mathcal C_A$, $(I,J)$ induces an interpretation $(I,J)_X \in \mathrm{Mod}_{\mathcal E}(X).$ We say that $(I,J)$ \emph{satisfies} a constraint $\rho = (X,C)$, written as $(I,J)\models \rho$, if $(I,J)_X$ satisfies $C$.
\end{definition}

We denote by $\mathrm{Agents}(A,\mathcal{E})$ the class of all agents of $A$ in $\mathcal{E}$.

Therefore, objects in $Agents(A,\mathcal{E})$ represent different concrete implementations of the same architectural specification and semantic universe of implementation, while morphisms correspond to translations and transformations between such implementations.

\begin{definition}[Category of Agents]
Let $A$ be an architecture and let $\mathcal{E}$ be the implementation category.
We define $Agents(A,\mathcal{E})$ as the category whose objects are admissible agents implementing $A$ in $\mathcal{E}$:

\begin{itemize}
    \item objects: agents $(I,J)$;
    \item morphisms $(\eta,\theta):(I,J)\to(I',J')$ given by
    monoidal natural transformations
    \[
    \eta:I\Rightarrow I',
    \qquad
    \theta:J\Rightarrow J'.
    \]
\end{itemize}
\end{definition}

Objects represent different realizations of the same architecture in $\mathcal{E}$, while morphisms correspond to transformations between such implementations.

\subsection{Reindexing and admissibility}
An architecture does not determine a unique agent, but rather a space of \emph{admissible realizations}, that is, implementations of the syntactic and knowledge layer satisfying both the structural interface and the constraint layer. This opens up the possibility for formalizing how agents with the same, or different architectures, are related. In this regard, architecture morphisms naturally induce translations between agents implemented on different architectures. Intuitively, if an architecture $A$ can be mapped into another architecture $B$, then, an agent implementing $B$ can be interpreted as an agent of $A$ by transporting its implementation along this architectural map, but this does not mean that this reinterpreted $B$-agent is admissible within architecture $A$. This appendix refines the categorical organization of such agents, extending the previous formulation to incorporate constraint satisfaction.

Let  $F : A \to B$ be a morphism in $\mathbf{ArchAgents}$.
At the syntactic and knowledge level, $F$ induces a reindexing operation by precomposition:
\[
F^\ast(I_B,J_B) := (I_B \circ F_G,\; J_B \circ F_{\mathsf{Know}}).
\]

Thus, an agent implementing the architecture $B$ can be reinterpreted as an agent of $A$ by translating both the architectural structure and the knowledge structure along the morphism $F$.
Intuitively, if we have implemented an agent in architecture $B$ and implementation $\mathcal{E}$, we can generate agents in architecture $A$ with the same implementation if we know the morphism that translate $A$ into $B$ ($f:A \to B$). For instance, building the syntactic $I_A$ will be done by first translating it into the corresponding syntactic part of $B$, and then applies the implementation provided by the
$B$-agent. Thus,
\[
I_A(X)=I_B(f_{\mathcal{G}}(X)).
\]
The same its done with $J$. However, in the presence of constraints, reindexing is no longer automatically admissibility-preserving.

\begin{remark}[Constraint transport along architecture morphisms]
Let $\rho = (X,P)$ be a constraint in $A$. The morphism $F$ induces a constraint
\[
F(\rho) := (F_C(X), F^\sharp(P))
\]
in $B$.
\end{remark}

\begin{definition}[Admissibility-preserving morphism]
An architecture morphism $F : A \to B$ is said to be \emph{admissibility-preserving} if for every admissible agent $(I_B,J_B)$ of $B$, the reindexed agent $F^\ast(I_B,J_B)$ is admissible in $A$.
\end{definition}

This holds whenever constraint satisfaction is stable under pullback along $F$, i.e.
\[
(I_B,J_B) \models F(\rho)
\;\;\Rightarrow\;\;
F^\ast(I_B,J_B) \models \rho.
\]

\paragraph{Interpretation.}
Not every architectural translation preserves admissibility: some morphisms forget constraints, weaken them, or change their scope. This reflects the fact that architectural abstraction and refinement interact non-trivially with semantic requirements.

\subsection{The indexed structure of admissible agents}

Fix an implementation category $\mathcal{E}$.

The assignment
\[
A \mapsto Agents(A,\mathcal{E})
\]
together with admissibility-preserving reindexing functors defines a contravariant pseudofunctor
\[
Agents(-,\mathcal{E}) : \mathbf{ArchAgents}^{op} \to \mathbf{Cat}.
\]

\begin{definition}[Total category of admissible agents]
The total category of admissible agents (for the implementation category
$\mathcal{E}$) is given by the Grothendieck construction
\[
Agents(\mathcal{E}) := \int_{A \in \mathbf{ArchAgents}} Agents(A,\mathcal{E}).
\]
\end{definition}

An object is a pair $(A,(I,J))$ where $(I,J)$ is an admissible agent of $A$.

A morphism
\[
(A,(I,J)) \to (B,(I',J'))
\]
consists of:
\begin{itemize}
    \item an architecture morphism $F : A \to B$;
    \item a monoidal natural transformation
    \[
    \eta : (I,J) \Rightarrow F^\ast(I',J'),
    \]
\end{itemize}
such that both agents involved are admissible.

\subsection{Fibred structure and constraints}

\begin{thm}
The projection
\[
p : Agents(\mathcal{E}) \to \mathbf{ArchAgents}
\]
defines a fibration over the subcategory of admissibility-preserving morphisms.
\end{thm}

\begin{proof}[Sketch]
The classical Grothendieck construction yields a fibration for the underlying pseudofunctor. The restriction to admissible agents requires stability of constraint satisfaction under reindexing, which holds precisely along admissibility-preserving morphisms.
\end{proof}

\paragraph{Conceptual consequence.}
The fibrational structure separates three levels:
\begin{itemize}
    \item architectural structure (base category),
    \item admissible semantic realizations (fibres),
    \item transport of implementations under architectural change (cartesian liftings).
\end{itemize}

Crucially, the constraint layer acts as a \emph{selection principle} on fibres, carving out admissible subcategories that are not invariant under arbitrary architectural morphisms.

\begin{remark}[Non-fibrational behaviour in general]
If arbitrary morphisms of $\mathbf{ArchAgents}$ are allowed, the projection need not define a fibration, reflecting that constraint satisfaction is not functorial in general. This is a feature rather than a limitation: it encodes the fact that not all architectural transformations preserve semantic validity.
\end{remark}

%% file: Parts/5_Properties.tex
\section{Properties of Architectures and Agents}
\label{sect:properties}
\paragraph{\Large Disclaimer.}
\textbf{\large The formalization presented in this section is deprecated and will be revised in future versions of this framework. 
Nevertheless, the underlying conceptual ideas and structural intuitions remain valid and will be preserved, albeit 
reformulated within a more robust and coherent mathematical setting.
}\\

Architectures specify the structural and informational laws of an agent
type, while agents instantiate these laws through concrete semantic
models and algorithms.
Accordingly, architectural properties concern invariants of the
architectural presentation itself, whereas semantic properties of agents
concern the behaviour of particular implementations and must be
established by mathematical or empirical certification.
This section formalizes both notions.

\subsection{Structural Properties of Architectures}

Let $A = (H_A,Info_A,\Omega_A)$ be an architecture.  
Recall that $\mathcal H_A$ is a hypergraph category freely generated by a
set of types and generators, modulo purely structural equations.
Every architectural diagram is therefore a morphism of $\mathcal H_A$,
constructed compositionally from these generators.

\begin{definition}[Structural property]
A structural property of an architecture $A$ is a judgment
\[ H_A \vdash \varphi, \]
where $\varphi$ is an equation, commutation law, or structural predicate
between string diagrams of $H_A$, derivable using only:
\begin{itemize}
  \item the equations defining the presentation of $\mathcal H_A$, and
  \item the axioms of hypergraph categories (monoidality, symmetry, units,
  Frobenius laws, etc.).
\end{itemize}

\end{definition}

Structural properties depend exclusively on:
\begin{itemize}
    \item the admissible generators and types of the architecture,
    \item the algebraic relations imposed between them, and
    \item the universal equational theory of hypergraph categories.
\end{itemize}

Typical examples of structural properties include:
\begin{itemize}
    \item existence or absence of feedback loops;
    \item reachability or accessibility relations between types;
    \item factorability or decomposability of diagrams;
    \item invariance of wiring patterns under Frobenius operations;
    \item equivalence of alternative architectural decompositions.
\end{itemize}

Structural properties are established by diagrammatic reasoning, that is,
by transforming string diagrams using the allowed equational rules or by
combinatorial inspection of their wiring structure.

Hypergraph categories provide a fully compositional equational calculus. 
That is,  if each component satisfies a structural property, and if the structural equations of the architecture are stable under composition and tensoring, then the property automatically lifts to any composite diagram.
Formally, given generators $g_1,\dots, g_n$, if
\[H_A \vdash \varphi(g_1), \dots, H_A \vdash \varphi(g_n),\]
and if $\varphi$ is preserved by monoidal composition and categorical
composition, then, for any diagram $d$ constructed from the generators
$g_i$, we have
\[ H_A \vdash \varphi(d). \]

Structural properties are preserved by morphisms of architectures:
if $F : A \to B$ is an architecture morphism and $H_A \vdash \varphi$,
then
\[ H_B \vdash F(\varphi).\]
Thus, structural reasoning is stable under refinement, embedding, or
translation of architectures.

\begin{proposition}
If $F : A \to B$ is a morphism in $\mathbf{ArchAgents}$ and
$H_A \vdash \varphi$, then $H_B \vdash F(\varphi)$.
\end{proposition}
Thus architectural properties are functorially stable under refinement
or translation of architectures.

\subsection{Informational Properties of Architectures}

Structural properties characterize the purely syntactic organization of
an architecture.  However, many fundamental differences between agent
types arise not from wiring alone, but from the way information is
handled, encapsulated, and propagated across architectural components.
These aspects are captured by informational properties.

Let
\[ A = (H_A,Info_A,\Omega_A)\]
be an architecture. Informational properties are properties of the
information category $Info_A$ and of the architectural
interpretation functor $\Omega_A$, and therefore constrain
all agents instantiating the architecture.

\begin{definition}[Informational Property]
An informational property of an architecture $A$ is a predicate
\[ \Psi(Info_A,\Omega_A)\]
that depends only on:
\begin{itemize}
  \item the categorical structure of $Info_A$, and
  \item the way architectural generators and diagrams of $H_A$
  are interpreted as morphisms in $Info_A$.
\end{itemize}
Informational properties do not depend on any particular semantic model,
learning algorithm, or agent implementation.
\end{definition}

\subsection{Semantic Properties of Agents}

Structural reasoning does not capture the behavioral or algorithmic
properties of particular agents.  
Such properties depend on the semantic interpretation
$F : H_A \to \mathbf{Sys}$ and, therefore, require additional information.
We model this information by attaching proofs to agent's implementations.
Semantic properties never modify $H_A$ or its relations.  
They attach behavioural guarantees to agents, not to
architectures, enabling compatibility with classical proofs not expressed
in categorical terms.

We give a mathematical account of semantic properties and
their certificates. The presentation is intentionally modula. That is, the
logical language used to express semantic properties is abstracted via
the notion of an institution, whereas certificates follow a
proof-carrying style and are instantiated within the semantic interpretation
of an architecture.

\begin{definition}[Institution]
An institution $\mathcal{I}=(\mathbf{Sign},\mathbf{Sen},\mathbf{Mod},\models)$
consists of:
\begin{itemize}
  \item a category $\mathbf{Sign}$ of signatures;
  \item a functor $\mathbf{Sen}:\mathbf{Sign}\to\mathbf{Set}$ assigning
    to each signature $\Sigma$ a set of sentences $\mathbf{Sen}(\Sigma)$;
  \item a functor $\mathbf{Mod}:\mathbf{Sign}^{op}\to\mathbf{Cat}$
    assigning to $\Sigma$ a category $\mathbf{Mod}(\Sigma)$ of models;
  \item for every signature $\Sigma$ a satisfaction relation
    $\models_\Sigma\subseteq \mathbf{Mod}(\Sigma)\times\mathbf{Sen}(\Sigma)$
    such that satisfaction is invariant under signature morphisms
    (the usual institution satisfaction condition).
\end{itemize}
\end{definition}
\paragraph{Intuitive explanation of Institution}
An institution abstracts the notion of a “logical system’’ away from any particular syntax or semantics. It does this by separating four components:
\begin{itemize}
    \item Signatures describe the symbols available for forming expressions (types, operations, predicates, state spaces, transition operators, etc).
    \item Sentences are the well-formed statements that can be written using such symbols (equations, inequalities, temporal properties, convergence claims, etc.).
    \item Models provide concrete interpretations of the symbols in a signature (a specific MDP, a particular transition kernel, an operator implementing a learning update, etc.).
    \item Satisfaction is the relation specifying when a model makes a sentence true.

\end{itemize}
The key feature is that institutions are logic-independent, that is, the satisfaction relation is stable under change of signature, and no assumption is made about the concrete syntax or proof system.
This makes institutions perfect for interfacing external theorems proved in any mathematical framework, with our agent semantics. That is, a theorem is represented through its signature and statements, while its proof lives outside the institution and agents contribute only the models that instantiate the signature so that the theorem becomes applicable.

\begin{definition}[Theorem / Theorem-signature]
A theorem $T$ is a triple $(\Sigma_T,\; \Gamma \vdash \varphi)$
where $\Sigma_T\in\mathbf{Sign}$ is a signature, $\Gamma\subseteq\mathbf{Sen}(\Sigma_T)$
is a (finite) set of premises (hypotheses) and $\varphi\in\mathbf{Sen}(\Sigma_T)$ is the conclusion. We say $T$ has signature $\Sigma_T$.
\end{definition}

\begin{remark} 
Every theorem $T=(\Sigma_T,\Gamma\vdash\varphi)$ used in our framework  is assumed to be backed by an external mathematical proof establishing 
the semantic entailment $\Gamma \models_{\Sigma_T} \varphi$. 
This proof may exist in any logical or mathematical setting (analysis, 
probability, optimization, control theory, category theory, etc.) and is not represented 
inside the institution. 
The institution retains only the abstract logical shape of the theorem, 
while agents provide the semantic models that instantiate their signature. 
Whenever the semantic interpretation of an agent satisfies the hypotheses  $\Gamma$, the conclusion $\varphi$ is inherited automatically by soundness of the external proof.
\end{remark}

Let $A=(H_A,D_A)$ be an architecture and $F:H_A\to\mathbf{Sys}$ an
$A$-agent. Assume that $\mathbf{Sys}$ is equipped with an institution
$\mathcal{I}$, or that there is a faithful embedding of the semantic
universe of $\mathbf{Sys}$ into $\mathcal{I}$, so that agents admit
models in $\mathbf{Mod}(\Sigma)$ for appropriate $\Sigma$.\\

\begin{definition}[Instantiation / signature morphism]
A signature instantiation of a theorem-signature $\Sigma_T$
into the agent $F$ is a signature morphism
$\tau : \Sigma_T \longrightarrow \Sigma_A$ in $\mathbf{Sign}$,
where $\Sigma_A$ is a signature for which the agent $F$ provides a
concrete model $M_F \in \mathbf{Mod}(\Sigma_A)$.
The instantiation $\tau$ maps abstract symbols of $T$ to concrete
objects in the semantic vocabulary of $F$.
\end{definition}

\begin{definition}[Semantic certificate]
Let $T=(\Sigma_T,\Gamma\vdash\varphi)$ be a theorem in an institution
$\mathcal{I}$. Let $F:H_A\to\mathbf{Sys}$ be an $A$-agent and let
$\tau:\Sigma_T\to\Sigma_A$ be a signature instantiation with associated
agent-model $M_F\in\mathbf{Mod}(\Sigma_A)$. A semantic certificate
for the claim ``$F$ satisfies the conclusion of $T$ under $\tau$'' 
is a triple
\[cert = \big(T, \tau, evds\big)\]
where:
\begin{enumerate}
  \item $\tau:\Sigma_T\to\Sigma_A$ is the signature instantiation
          mapping the abstract symbols of $T$ to the concrete semantic 
          components of the agent.
  \item $evds$ is a verification artifact that establishes the semantic 
          validity of the instantiated hypotheses: 
          \[M_F\models_{\Sigma_A}\tau(\Gamma)\]
    The artifact may take different forms, such as:
    \begin{itemize}
        \item a machine-checkable proof term (e.g.\ Coq/Lean proof),
        \item explicit witness objects and decidable checks  (e.g.\ verifying stochasticity, contraction constants, Lipschitz bounds, structural constraints),
        \item a human-readable mathematical argument that clearly identifies why each assumption in $\Gamma$ holds for $M_F$.
    \end{itemize}
\end{enumerate}
We write $\mathsf{check}(\mathsf{cert},M_F)$ for the (computable)
verifier which validates $\mathsf{evidence}$ against $M_F$.
\end{definition}

\begin{remark}
The certificate does not contain a proof of the theorem $T$ itself. The 
theorem is assumed to be justified by an external mathematical proof.  
The evidence $evds$ only establishes that the semantic interpretation 
$M_F$ of the agent satisfies the hypotheses $\Gamma$ under the chosen 
instantiation $\tau$. Once this is verified, the soundness of the external proof ensures that $M_F$ also satisfies the conclusion $\varphi$.
\end{remark}

\begin{proposition}[Soundness of certificate transfer / theorem instantiation]
Let $T=(\Sigma_T,\Gamma\vdash\varphi)$ be a theorem of the institution
$\mathcal{I}$ and let $\tau:\Sigma_T\to\Sigma_A$ be a signature
morphism. Assume:
\begin{enumerate}
  \item There exists a (trusted) external proof object $\pi$ certifying
        the validity of $T$, i.e.\ $\pi$ witnesses that for every
        $\Sigma_T$-model $M$, if $M\models\Gamma$ then $M\models\varphi$.
        (The proof does \emph{not} involve the agent.)
  \item The evidence $evds$ verifies that the agent-model
        $M_F \in \mathbf{Mod}(\Sigma_A)$ satisfies the instantiated
        premises, that is that $\mathsf{check}((T,\tau,evds),M_F)=\mathrm{true}$, and therefore, the evidence establishes $M_F\models_{\Sigma_A}\tau(\Gamma)$.
\end{enumerate}
Then, by the soundness of the logic of $\mathcal{I}$ and the invariance of
satisfaction under signature morphisms,
\[M_F \models_{\Sigma_A} \tau(\varphi).\]

In plain words: if the theorem is valid in general and the agent satisfies the
instantiated hypotheses, then the agent also satisfies the instantiated
conclusion.
\end{proposition}

\begin{proof}[Proof Sketch]
The hypothesis (1) guarantees that $\Gamma\models_{\Sigma_T}\varphi$ in
the logical system (soundness). By signature morphism $\tau$ and the
institution satisfaction condition, if $M_F\models_{\Sigma_A}\tau(\Gamma)$
then $M_F\models_{\Sigma_A}\tau(\varphi)$. The verifier in (2) ensures
that the premises hold in the concrete model, so the conclusion follows.
\end{proof}

\begin{definition}[Semantic property of an agent]
Let $T=(\Sigma_T,\Gamma\vdash\varphi)$ be a theorem in the ambient
institution, let $\tau:\Sigma_T\to\Sigma_A$ be an instantiation, and
let $F:H_A\to\mathbf{Sys}$ be an $A$-agent with semantic model
$M_F\in\mathbf{Mod}(\Sigma_A)$.
A semantic property of $F$ is a certified theorem instance
\[P = (T,\tau,cert) \]
where $cert=(T,\tau,evds)$ is a semantic certificate showing that
the instantiated hypotheses hold for $M_F$:
\[ \mathsf{check}(cert, M_F)=\text{true}.\]

In this case we say that ``The agent $F$ satisfies the property $P$'', and the satisfaction condition guarantees
\[ M_F \models_{\Sigma_A} \tau(\varphi).\]
\end{definition}

\subsubsection{Composition of certificates (modularity)}
The monoidal structure on $\mathbf{Sys}$ and the functorial nature of
agents allow composition of certificates. We state a sufficient condition.

\begin{definition}[Compositional certificate operator]
Let $cert_1=(T_1,\tau_1,\mathsf{e}_1)$ and
$cert_2=(T_2,\tau_2,\mathsf{e}_2)$ be certificates for two
subagents \(F_1\) and \(F_2\) whose monoidal composition yields the
composite agent \(F=F_1\otimes F_2\). A compositional operator
\[ \odot : \mathsf{Cert}\times\mathsf{Cert} \longrightarrow \mathsf{Cert} \]
produces \(cert=cert_1\odot cert_2\) whenever
the following hold:
\begin{enumerate}
  \item The signatures $\Sigma_{T_1},\Sigma_{T_2}$ and their instantiations $\tau_1,\tau_2$ can be coherently combined into a signature $\Sigma_T$ and an instantiation $\tau:\Sigma_T\to\Sigma_{F}$ for the composite agent (this is typically given by the coproduct or tensoring of signatures and the obvious induced signature morphism).
  \item The premises required by the composed theorem $T$ are covered
    by the union of premises validated by $\mathsf{e}_1$ and $\mathsf{e}_2$, possibly together with additional (verifiable) interface assumptions.
  \item There is a mechanizable construction producing evidence
    $\mathsf{e}$ for the composite that the verifier accepts:
    $\mathsf{check}((T,\tau,\mathsf{e}),M_F)=\mathrm{true}$.
\end{enumerate}
\end{definition}

\begin{lemma}[Preservation under reindexing]
Let $f:A\to B$ be an architecture morphism with induced hypergraph
functor $H(f):H_A\to H_B$. Let $G\in Agents(B)$ be a $B$-agent
with certificate $cert_G=(T,\tau_G,\mathsf{e}_G)$. If the
certificate is stable under the signature translation induced by $f$
and $\mathsf{check}(\mathsf{cert}_G,M_G)=\mathrm{true}$, then the
reindexed agent $f^*(G)=G\circ H(f)$ inherits a transported certificate
$\mathsf{cert}_{f^*(G)}$, obtained by precomposing $\tau_G$ with the
induced signature morphism, that is valid for $f^*(G)$.
\end{lemma}

\begin{proof}[Proof sketch]
The architecture morphism induces a translation of semantic vocabulary
(i.e. a signature morphism) compatible with the instantiation $\tau_G$.
Precomposing yields an instantiation for the reindexed agent; since the
evidence $\mathsf{e}_G$ validates the premises in the original model,
and the signature translation respects satisfaction, the transported
certificate verifies the premises for the reindexed model as well.
\end{proof}

\subsubsection{Practical verification workflow.}
Given a claim ``agent $F$ satisfies conclusion of theorem $T$'' we
require the following manifest:
\begin{enumerate}
  \item the reference theorem $T=(\Sigma_T,\Gamma\vdash\varphi)$ and,
    optionally, the trusted external proof $\pi$ of $T$ in the chosen logic;
  \item a signature instantiation $\tau:\Sigma_T\to\Sigma_A$ linking the
    theorem symbols to the semantic vocabulary of $F$,
  \item evidence $\mathsf{e}$ witnessing  $M_F\models_{\Sigma_A}\tau(\Gamma)$;
  \item the output certificate $\mathsf{cert}=(T,\tau,\mathsf{e})$ and the result of $\mathsf{check}(\mathsf{cert},M_F)$.
\end{enumerate}
If $\mathsf{check}$ returns $\mathrm{true}$, the system accepts the
semantic property $M_F\models\tau(\varphi)$ for $F$.

\begin{remark}[Design choices]
Some design choices:
    \begin{itemize} 
      \item The use of institutions ensures that the theorem $T$ need not be reformulated in a canonical logical syntax: only its signature and premises must be instantiable.
      \item Proof-carrying style allows to verify if the implementation of the agent still fulfills the hypotheses of the theorem, and therefore the conclusion remains true for the agent,
      \item The described mechanism integrates with the Grothendieck
        fibration $Agents \to \mathbf{ArchAgents}$: certificates live in the fibre (they are attached to particular agent implementations) and are transported along reindexing functors as described.
    \end{itemize}
\end{remark}

This ensures that:
\begin{itemize}
    \item existing/classical proofs can be used unchanged;
    \item architectural reasoning remains diagrammatic and structural;
    \item semantic properties remain portable across architectures via certificates.
\end{itemize}

 \begin{figure}
     \centering
     \includegraphics[width=\linewidth]{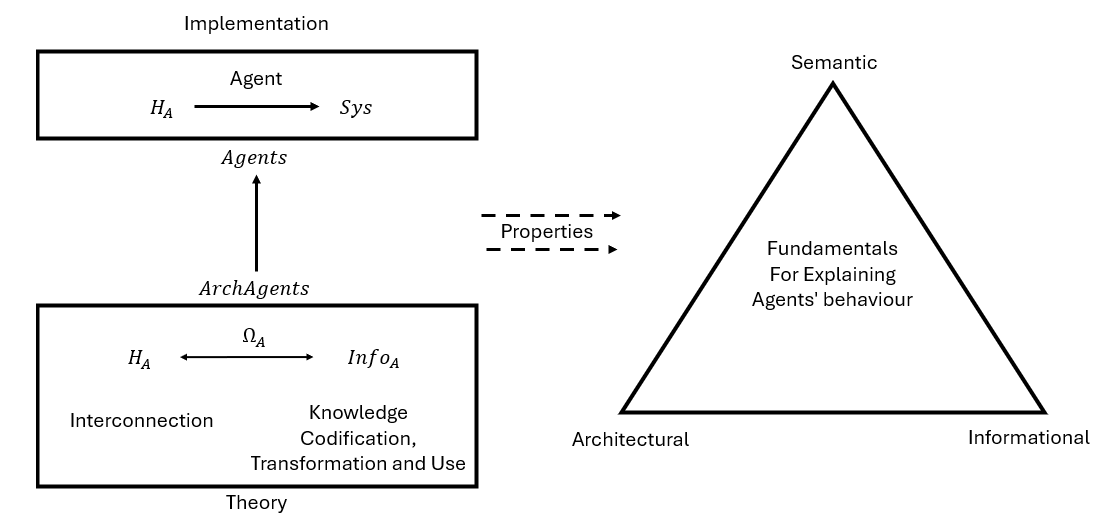}
     \caption{Framework map}
     \label{fig:Framework map}
 \end{figure}

%% file: Parts/6_Architectures_Cases_Studies.tex
\section{Case Studies: From the RL to the SBL Architecture}
\label{sect:CaseStudies}
In this section we present a sequence of agent architectures, each formally defined within the categorical framework introduced above. Rather than proposing new learning algorithms, our goal here is to illustrate how different learning paradigms can be characterized, compared, and analyzed at the architectural level, independently of their concrete implementations.
The examples are organized progressively. We begin with classical Reinforcement Learning as a minimal and widely adopted baseline paradigm, and subsequently enrich its syntactic structure to capture additional cognitive capabilities. Throughout, we place particular emphasis on how each architecture handles information and knowledge, That is, how information flows, how it is stored, how it is updated, and whether it can be structured, factored, or reused. 
This emphasis is motivated by a central architectural challenge in learning agents, namely, avoiding catastrophic forgetting/interference as well as enabling continual learning. From an architectural standpoint, this requires mechanisms for modular, hierarchical, and factorized representations of knowledge. The following examples should therefore be read not only as isolated case studies, but as successive points along a spectrum of increasing informational structure.

\subsection{Case Study I: Reinforcement Learning Architecture}
We begin by illustrating how the classical Reinforcement Learning (RL) architecture can be expressed as an object in the category $\mathbf{ArchAgents}$. This example serves as a baseline architecture against which more expressive learning paradigms will be later compared. In section \ref{app:RL_agent} we also present an agent realizing the RL architecture.

At an architectural level, RL is characterized by a \textbf{flat control structure and a single, globally coupled information flow}. All persistent knowledge acquired through interaction with the environment is aggregated into a \textbf{single parametric carrier, and no explicit internal structure} is imposed on this knowledge.

\subsubsection{RL syntactic layer.}
The syntactic layer $Syn_{RL}$ is freely generated by the syntactic presentation $RL=(STypes_{RL},SGen_{RL},SEq_{RL})$. $STypes_{RL}$ is defined by the following types:
\begin{itemize}
    \item $S$ the state type
    \item $A$ the action type
    \item $E$ the experience type, that involve information about the state, action and reward. This type could be replaced by the reward type $R$ and use a tuple/composition of types in the diagram
    \item $\Theta^s$ the function/model/parameters type representing the "engine" that the agent updates and uses
\end{itemize}
All tensor expressions over these types are available via the symmetric monoidal structure.
The primitive syntactic \textbf{generators} of $Syn_{RL}$ are:
\[
\begin{aligned}
&\mathsf{Policy} : S \otimes \Theta^s \longrightarrow A, \\
&\mathsf{EnvInteraction} : S \otimes A \longrightarrow E, \\
&\mathsf{Update} : \Theta^s \otimes E \longrightarrow \Theta^s.
\end{aligned}
\]
These generators represent abstract syntactic roles: policy
calculation, interaction with the environment, and internal state update. In this case we leave $Eq_{RL}$ empty by now.
Figure \ref{fig:RL string diagram} depicts the corresponding syntax pattern representation $\mathcal{G}_{RL}$. We took inspiration from the general string diagram in Figure 6 from \cite{Hedges2024ReinforcementLI}.

\begin{figure}[!ht]
    \centering
    \includegraphics[width=0.5\linewidth]{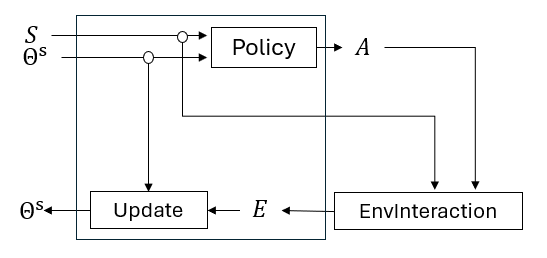}
    \caption{RL string diagram.}
    \label{fig:RL string diagram}
\end{figure}

\subsubsection{RL Knowledge layer}
The knowledge layer $Know_{RL}$ is freely generated by the knowledge presentation $K_{RL}=(KTypes_{RL},KGen_{RL},KEq_{RL}$, where:
\begin{itemize}
    \item $KTypes_{RL}=\{\Theta^k, E^k\}$
    \item $KGen_{RL}=\{Upd:\Theta^k \otimes E^k\rightarrow \Theta^k\}$
    \item $KEq_{RL}$ is empty (apart from the basic frobenius structure)
\end{itemize}

\subsubsection{RL Relational Interface}

The interaction between the syntax layer and the knowledge layer
is mediated by a profunctor
\[
\Phi_{RL} : \mathcal{G}_{RL}^{op} \times Know_{RL} \longrightarrow \mathbf{Set},
\]
which specifies how syntactic types are related to available knowledge
carriers.
At the level of objects, $\Phi_{RL}$ is defined as follows:
\begin{itemize}
    \item $\Phi_{RL}(\Theta^s,\Theta^k)=\{\star\}$, $\Phi_{RL}(E,E^k)=\{\star\}$
    \item $\Phi_{RL}(X,\Theta^k)=\varnothing$ for all $X \in STypes_{RL}$ with $X \neq \Theta^s$.
\end{itemize}

No other object-level relations are specified. The action of $\Phi_{RL}$ on morphisms is the canonical one induced by precomposition in $Syn_{RL}$ and postcomposition in $Know_{RL}$.
Intuitively, this profunctor expresses the fact that the only syntactic entities that store persistent knowledge is the parameter type $\Theta^s$ and $E$. All other syntactic types (such as $S$, $A$) are purely operational and do not correspond to stable knowledge representations.

\begin{table}[h]
\centering
\renewcommand{\arraystretch}{1.25}
\setlength{\tabcolsep}{18pt}

\begin{tabular}{|c|c|c|c|c|}
\hline
\textbf{Types} & $S$ & $A$ & $E$ & $\Theta^s$ \\
\hline
$E^k$ & \xmark & \xmark & \cmark & \xmark \\
\hline
$\Theta^k$ & \xmark & \xmark & \xmark & \cmark \\
\hline
\end{tabular}
\caption{Visualization as a table of the support of the RL relation profunctor}
\end{table}

\subsubsection{RL constraints}

Standard RL is not only characterized by a perception-action-update loop, but also by additional constraints, namely: 
\begin{enumerate}
    \item the persistent knowledge state must be interpretable as an evaluative object over expected cumulative returns (e.g. a value or action-value structure),
    \item learning updates should be Bellman-compatible, in the sense of targeting temporally consistent fixed points,
    \item policy selection should be appropriately guided by such evaluative knowledge, and
    \item the environment interface is typically assumed to satisfy a Markovian transition structure.
\end{enumerate}   
Moreover, RL often carries further ontological assumptions on its types, such as whether the state space is atomic or factorized, whether observations coincide with states, or whether the action space is discrete or continuous.
These restrictions are essential for distinguishing RL from more general adaptive control loops. These constraints would be formalized as follows within the framework

\paragraph{Value/Action Representability}

This constraint captures the defining semantic commitment of RL that  
the internal knowledge state must encode evaluative information 
about future interaction outcomes. 
Without this restriction, the 
architecture degenerates into a generic adaptive control loop.
Formally, the value representability constraint enforces that the 
distinguished knowledge type $\Theta^k$ admits only realizations 
that can be interpreted as expected cumulative return.
Specifically, the value representability constraint is given by 
\[
\rho_{\mathrm{val}}^{RL}
=
(\Theta^k,\; P_{\mathrm{val}}),
\quad
P_{\mathrm{val}} \in \mathcal P_{RL}(\Theta^k),
\]
where $\Theta^k$ is the knowledge type representing evaluative content.
The property $P_{\mathrm{val}}$ requires that any admissible realization
of $\Theta^k$ is representable as an evaluative object over future returns.
Concretely, for a semantic realization $(I,J)$, the interpretation
\[
J(\Theta^k)
\]
must admit a structure equivalent to one among:
\[
V : S \to \mathbb{R},
\qquad
Q : S \times A \to \mathbb{R},
\]
or, more generally, any object encoding expected discounted cumulative
return up to semantic equivalence.
Thus, admissible realizations of $\Theta^k$ are restricted to lie in the
full subcategory of evaluative representations. The type of this constraint under the typology we suggested in the remark \ref{remark:Constraint typology} would be \textbf{membership-in-class / representability}.

\paragraph{Bellman-type Consistency}

This constraint enforces temporal coherence of evaluative knowledge. 
While the previous constraint restricts the representational form 
of $\Theta^k$, the Bellman constraint imposes a dynamical condition. 
That is, valid evaluative representations must be stable under a Bellman-type 
update operator.
This is the key property distinguishing RL from arbitrary predictive 
or descriptive learning systems.
The Bellman consistency constraint is given by
\[ \rho_{\mathrm{Bell}}^{RL} = (X_{\mathrm{Bell}},\; P_{\mathrm{Bell}}),
\quad
P_{\mathrm{Bell}} \in \mathcal P_{RL}(X_{\mathrm{Bell}}),
\]
where the scope
\[ X_{\mathrm{Bell}} = \{\Theta^k,\, Upd\}\]
may be extended to include $S, A, E$ and the and the relevant interface relations through $\Phi_{RL}$, insofar as they contribute to the evaluative semantics of $\Theta^k$.
The property $P_{\mathrm{Bell}}$ requires that, for any realization $(I,J)$:

\begin{itemize}
    \item the interpretation $J(\Theta^k)$ carries an induced operator
    \[ \mathcal B : J(\Theta^k) \to J(\Theta^k),\]
    corresponding to a Bellman-type operator;

    \item the optimal evaluative content is a fixed point of this operator,
    \[ V = \mathcal B(V)
    \quad\text{or}\quad
    Q = \mathcal B(Q);
    \]

    \item the realization of the update generator
    \[ J(Upd) \]
    induces a learning dynamics compatible with $\mathcal B$, in the sense
    that its trajectories converge (in the intended semantics) towards a
    Bellman-consistent solution, typically satisfying
    \[ V^\ast = \mathcal B(V^\ast).\]
\end{itemize}

Thus, $P_{\mathrm{Bell}}$ constrains both the internal structure of
$\Theta^k$ and its dynamical interaction with the update mechanism.
The type of this constraint under the typology we suggested in the remark \ref{remark:Constraint typology} would be \textbf{fixed-point} or optionally \emph{optimality / extremality}.

\paragraph{Policy--Value Compatibility}

This constraint captures the functional role of evaluative knowledge 
in RL. Policies must be informed by value-like representations. 
Without this coupling, the evaluative structure becomes behaviorally 
irrelevant.
This constraint therefore enforces a dependency between the policy 
generator and the evaluative knowledge encoded in $\Theta^k$.
The RL policy guidance constraint is given by
\[ \rho_{\mathrm{pol}}^{RL} = (X_{\mathrm{pol}},\; P_{\mathrm{pol}}),
\quad
P_{\mathrm{pol}} \in \mathcal P_{RL}(X_{\mathrm{pol}}),
\]
where
\[
X_{\mathrm{pol}} = \{\mathrm{Policy},\, \Theta^s,\, \Theta^k\}.
\]

The property $P_{\mathrm{pol}}$ requires that, for any realization $(I,J)$,
the interpretation
\[
I(\mathrm{Policy}) : I(S) \otimes I(\Theta^s) \to I(A)
\]
is guided by evaluative knowledge obtained through the coupling between
$\Theta^s$ and $\Theta^k$ via the interface $\Phi_{RL}$:
\[\exists F \text{ such that } I(\mathrm{Policy}) = F(J(\Theta^k))\]

Concretely, the behavior induced by the policy must depend on a realization of $\Theta^s$ whose associated knowledge content encodes expected cumulative future reward, and must be preferentially aligned with it.
Typical admissible realizations include greedy, soft-greedy, or otherwise
value-derived policies, but the formulation remains abstract and
representation-independent.
The type of this constraint under the typology we suggested in the remark \ref{remark:Constraint typology} would be \textbf{interface compatibility / dependency}.

\paragraph{Markovian Transition Admissibility}
This constraint specifies the structural assumption on the environment 
interface underlying standard RL. That is, the interaction process must admit 
a Markovian factorization.
This distinguishes RL from more general history-dependent or 
non-Markovian interaction models.
The Markov constraint is given by
\[\rho_{\mathrm{Markov}}^{RL} = (X_{\mathrm{Markov}},\; P_{\mathrm{Markov}}),
\quad
P_{\mathrm{Markov}} \in \mathcal P_{RL}(X_{\mathrm{Markov}}),
\]
where
\[
X_{\mathrm{Markov}} = \{\mathrm{EnvInteraction},\, S,\, A,\, E\}.
\]

The property $P_{\mathrm{Markov}}$ requires that, for any realization $(I,J)$, the induced interaction semantics satisfies a Markovian factorization condition.
Specifically, the distribution of the next experience depends only on the
current state-action pair, i.e. the transition/reward structure factors
through the present configuration $(s_t,a_t)$, independently of prior
history except insofar as it is encoded in the current state.
This can be expressed as a conditional independence or factorization
property internal to the semantics of the interaction generator.
The type of this constraint under the typology we suggested in the remark \ref{remark:Constraint typology} would be \textbf{conditional independence / factorization}.

\paragraph{Ontological Constraints on Types}

In addition to operational constraints on dynamics and interaction, 
architectures may impose constraints on the internal structure of 
their types. These constraints express ontological commitments about 
how the relevant entities (states, observations, actions, rewards) 
are organized.
Crucially, such constraints do not merely restrict admissible 
implementations, but induce systematic variations of the architecture 
itself. In this sense, they generate families of related architectures 
rather than refinements of a fixed one.
For instance:
\begin{itemize}
    \item atomicity or factorization of the state type,
    \item the observation type may coincide with the state type (fully observable RL) or be distinct (partially observable variants),
    \item discreteness or continuity of action types,
    \item scalar or structured reward spaces.
\end{itemize}

Such assumptions correspond to constraints of the form
\[\rho = (X, P), \quad P \in \mathcal P_{RL}(X), \]
where $X$ includes the relevant scope, and $P$
specifies the required structural property (e.g. existence of a
factorization, product decomposition, or measurable structure).
These ontological constraints are essential for distinguishing
variants of RL architectures such as offline RL, modular RL, or factored RL and should be treated as first-class elements of the constraint system.
For instance, imposing a factorization constraint on the state type:
\[
S \cong S_1 \otimes S_2 \otimes \cdots \otimes S_n
\]

induces a factored variant of the RL architecture, in which both 
the transition dynamics and the value representation are expected 
to respect this decomposition.
This affects multiple architectural layers simultaneously:

\begin{itemize}
    \item \textbf{Syntax layer:} generators such as $\mathrm{EnvInteraction}$ 
    may factor or decompose across components;

    \item \textbf{Knowledge layer:} evaluative representations may decompose 
    additively or compositionally (e.g. factored value functions);

    \item \textbf{Constraints layer:} additional compatibility conditions 
    may be imposed to ensure that the dynamics preserves or exploits 
    the factorization.
\end{itemize}

Therefore, ontological constraints should be understood as 
\emph{architecture-generating constraints}.That is, rather than only selecting 
a subset of admissible agents within a fixed architecture, they 
define new architectural objects in $\mathbf{ArchAgents}$ which potentially could lead to the implementation of agents with new desirable properties.

\subsubsection{RL as object in ArchAgents}
Finally, the RL architecture is described as an object in
$\mathbf{ArchAgents}$ given by
\[
RL=(\mathcal{G}_{RL}, Know_{RL}, \mathcal{T}_{RL}).
\]
where $\mathcal{T}_{RL}=(\Phi_{RL}, \pi_{RL}:\mathcal{P}_{RL} \to \mathcal{C}_{RL},\mathcal{R}_{RL}=\{\rho_{val}^{RL},\rho_{Bell}^{RL},\rho_{pol}^{RL},\rho_{Markov}^{RL}\}$.

From an architectural perspective, classical Reinforcement Learning exhibits a highly centralized and undifferentiated treatment of information. All the persistent knowledge acquired through interactions with the environment is encoded into a \textbf{single parametric carrier $\Theta^k$}, which is both consumed by the policy to produce actions and updated as a result of experience.
This design has a number of architectural strengths. Firstly, it enforces a clear and simple information flow: experience is aggregated into
parameters, and parameters fully determine future behavior. Secondly, it supports continuous adaptation through repeated endomorphic updates of $\Theta^k$, thus, enabling incremental learning without requiring explicit memory management or representational commitments. Finally, the absence of internal structure in $\Theta^k$ makes the architecture broadly compatible with a wide range of concrete realizations.

However, the simplicity of $Know_{RL}$ also entails significant architectural
limitations. The architecture does not distinguish between different types or sources of knowledge, nor does it provide mechanisms for structuring, isolating, or selectively reusing information. All learned content is collapsed into $\Theta^k$, which acts as an informational bottleneck. As a consequence, Reinforcement Learning architectures lack explicit support for \textbf{modular knowledge, causal abstraction, contextual memory, or hierarchical reutilization}, all of which must be introduced, if at all, at an algorithmic/implementation rather than at the architectural level.
Thus, information transformations are constrained by the following architectural principles:
\begin{itemize}
  \item Information may be freely copied or discarded.
  \item Persistent information must be encoded into $\Theta^k$.
  \item Learning corresponds to transformations
  \[\Theta^k \longrightarrow \Theta^k\]
  representing parametric updates driven by experience.
\end{itemize}

\subsubsection{RL Architecture Properties}
Structural properties:
\begin{itemize}
    \item There exists, at most, one structurally distinct \textbf{feedback loop} involving knowledge: for all diagrams in $\mathcal{G}_{RL}$ they contain at most one loop that updates $\Theta^s$.
    \item \textbf{Indistinguishability} between types of $E$, that is, all experiences are indistinguishable from a syntactic level.
\end{itemize}
Informational properties:
\begin{itemize}
    \item \textbf{Closure of information}, that is, all the persistent information is encoded in $\Theta^k$.
    \item No existence of \textbf{knowledge modularity}, that is, the architecture can not express any type of decomposition, isolation or reusability of knowledge.
\end{itemize}

The limitations observed in the RL architecture are not merely algorithmic, but architectural in nature. In particular, RL provides no dedicated informational space for representing \textbf{causal relations} in the environment, nor does it distinguish between learning a policy for long-term reward optimization and learning a \textbf{structural model of the environment} itself. This limitations motivate the next Case study, in which causal structure is made explicit at the architectural level. In this regard, Causal Reinforcement Learning (CRL) extends the RL architecture by introducing separate informational components for causal variables and structural models,
thereby, \textbf{decoupling policy optimization from causal discovery and reasoning}.

\subsection{Case Study II: Causal Reinforcement Learning Architecture}

The Causal Reinforcememnt Learning (CRL) architecture extends the classical RL architecture by introducing an explicit internal representation of causal structure that allows the agent to reason about \textbf{interventions, counterfactuals, and causal dependencies}, rather than relying solely on associative experience. Instead of collapsing all learned knowledge into a single undifferentiated parameter carrier, CRL distinguishes between policy parameters and a \textbf{causal world model} that supports intervention-aware reasoning.
At an architectural level, CRL is characterized by the presence of the distinct causal model component and by the interpretation of actions as interventions on this model. However, the overall control structure remains globally coupled and similar to classical RL. That is, a model-mediated control loop structure, in which action selection and learning are conditioned not only on parametric policy knowledge but also on an explicit causal representation of the environment.

\paragraph{CRL syntactic layer.}
The syntactic layer $Syn_{CRL}$ is freely generated by the syntactic presentation $CRL=(STypes_{CRL},SGen_{CRL},SEq_{CRL})$. $Types_{CRL}$ is defined by the following type symbols:
\begin{itemize}
    \item $S$ the state type,
    \item $A$ the action type,
    \item $E$ the experience type,
    \item $\Theta_\pi^s$ the policy parameters type,
    \item $\Theta_{CS}^k$ the causal function/model/parameters type, representing the agent's internal causal representation of the environment.
\end{itemize}
All tensor expressions over these types are available via the symmetric monoidal structure.
The primitive syntactic generators of $Syn_{CRL}$ are:
\[
\begin{aligned}
&\mathsf{Policy} : S \otimes \Theta_\pi^s \otimes \Theta_{CS}^s \longrightarrow A, \\
&\mathsf{EnvInteraction} : S \otimes A \longrightarrow E, \\
&\mathsf{Do} : \Theta_{CS}^s \otimes A \longrightarrow \Theta_{CS}^s, \\
&\mathsf{PolicyUpdate} : \Theta_\pi^s \otimes \Theta_{CS}^s \otimes E \longrightarrow \Theta_\pi^s.\\
&\mathsf{CausalUpdate}: \Theta_{CS}^s \otimes E \longrightarrow \Theta_{CS}^s.
\end{aligned}
\]
These generators represent the following abstract syntactic roles:
\begin{itemize}
    \item $\mathsf{Policy}$: action selection conditioned on both parametric and causal knowledge,
    \item $\mathsf{EnvInteraction}$: interaction with the environment that produces the experience,
    \item $\mathsf{Do}$: causal intervention on the internal causal model,
    \item $\mathsf{CausalUpdate}$: learning and refinement of the internal causal model based on the experience,
    \item $\mathsf{PolicyUpdate}$: policy adaptation informed both by the experience and by the causal structure.
\end{itemize}

We do not impose additional equations beyond those required by the axioms of hypergraph categories; thus, $SEq_{CRL}$ remains empty.
The syntactic pattern $\mathcal{G}_{CRL}$ (Figure \ref{fig:CRL string diagram}) explicitly exhibits two coupled feedback loops: one over $\Theta_\pi^s$ and another one over $\Theta_{CS}^s$.

\begin{figure}[ht!]
    \centering
    \includegraphics[width=0.75\linewidth]{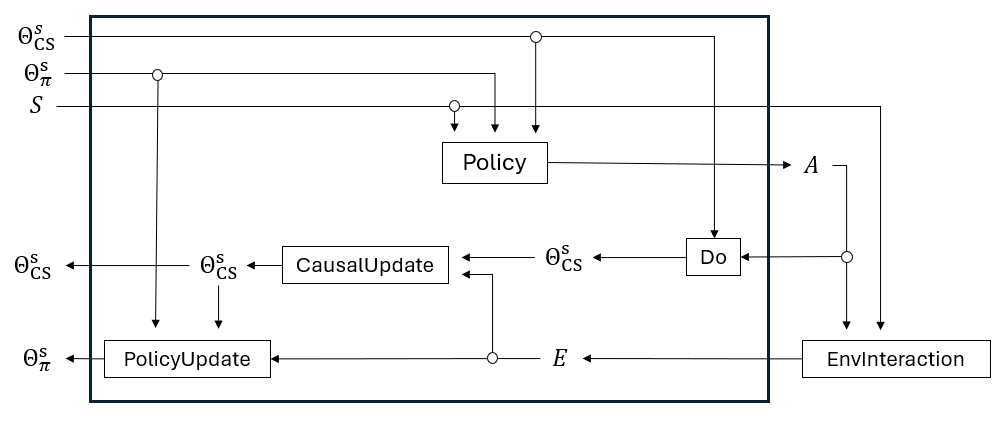}
    \caption{CRL string diagram.}
    \label{fig:CRL string diagram}
\end{figure}

\begin{remark}
The CRL architecture defined here should be understood as a general and minimal causal reinforcement learning architecture. In particular, no assumptions are made about the internal structure of the causal model $\Theta_{CS}$ nor the representation of the state type. 

More refined variants of CRL have been proposed in the literature, such as factorized CRL or CRL with latent variables 

In these architectures, the state is decomposed into multiple observed variables and the causal model explicitly represents factorized relations, as well as extensions that incorporate the explicit modeling and learning of latent or confounding variables. These approaches can be seen as architectural refinements of the present formulation, obtained by introducing \textbf{additional internal structure on $\Theta_{CS}^s$} rather than by modifying the core causal wiring pattern. For now, we maintain the focus on the general case study, since, as it can be noted, these variants suppose big changes in the definition of the syntax or the architectural information structure. 
\end{remark}

\paragraph{CRL knowledge layer.}
The knowledge layer $Know_{CRL}$ is freely generated by the knowledge presentation $Know_{CRL}=(KTypes_{CRL},KGen_{CRL},KEq_{CRL})$ where:
\begin{itemize}
    \item $KTypes_{CRL}=\{\Theta_\pi^k,\Theta_{CS}^k, E^k\}$
    \item $KGen_{CRL}=\{PolicyUpd:\Theta_{CS}^k\otimes \Theta_{\pi}^k \otimes E^k \rightarrow \Theta_{\pi}^k, \quad CausalUpd:\Theta_{CS}^k \otimes E^k \rightarrow \Theta_{CS}^k, \quad CausalIntervention:\Theta_{CS}^k \rightarrow \Theta_{CS}^k \}$
    \item $KEq_{CRL}$ is empty
\end{itemize}
While both $CausalUpd$ and $CausalIntervention$ are endomorphisms on $\Theta_{CS}^k$, they play quite distinct roles. That is, $CausalUpd$ represents learning-driven refinement from experience, whereas $CausalIntervention$ represents deliberate counterfactual or interventional modification used for decision-making.

\paragraph{Relational Profunctor for CRL.}
The interaction between the syntactic diagrams and the internal knowledge resources in CRL is specified by the following relational profunctor:
\[
\Phi_{CRL} : \mathcal{G}_{CRL}^{\mathrm{op}} \times \mathsf{Know}_{CRL}
\longrightarrow \mathbf{Set}.
\]

At the object level, this profunctor has non-empty support on the pairs
\[(\Theta_\pi^s,\Theta_\pi^k), \qquad (\Theta_{CS}^s,\Theta_{CS}^k), \quad (E,E^k)\]
indicating that the syntactic workflows may independently access and transform the policy knowledge and the causal knowledge.
This profunctorial structure makes explicit that CRL admits multiple,
non-collapsible knowledge access patterns, in contrast with the single
carrier architecture of classical RL.

\begin{table}[h]
\centering
\renewcommand{\arraystretch}{1.25}
\setlength{\tabcolsep}{18pt}

\begin{tabular}{|c|c|c|c|c|c|}
\hline
\textbf{Types} & $S$ & $A$ & $E$ & $\Theta^s_\pi$ & $\Theta^s_{CS}$ \\
\hline
$E^k$ & \xmark & \xmark & \cmark & \xmark & \xmark\\
\hline
$\Theta^k_\pi$ & \xmark & \xmark & \xmark & \cmark & \xmark\\
\hline
$\Theta^k_{CS}$ & \xmark & \xmark & \xmark & \xmark & \cmark \\
\hline

\end{tabular}
\caption{Visualization as a table of the support of the CRL relational interface profunctor}
\end{table}

\paragraph{CRL constraints}
For distinguish CRL from generic model-based or multi-module RL agents, standard CRL carries stronger restrictions, in addition to the same constraints as in RL,  namely:  
\begin{enumerate}
    \item the causal carrier must be admissibly realizable as a genuine causal world model rather than as an arbitrary predictive model,
    \item admissible realizations of $\mathsf{Do}$ should satisfy interventional semantics rather than mere observational conditioning,
    \item causal learning updates should preserve or refine the causal interpretability of the model and
    \item policy selection should depend non-trivially on causal knowledge, in particular through the evaluative use of interventional or counterfactual consequences for action selection.
\end{enumerate}

More generally, CRL also presupposes compatibility between the environment interface and an underlying causal generative structure, together with possible ontological assumptions about factorization, manipulability, observability, and latent variables.
These restrictions are essential to distinguish CRL from broader classes of structured or model-based RL agents, but in the present paper we does not formalize them yet under the developed framework

\paragraph{CRL as object in ArchAgents}
Finally, the CRL architecture is described by the following object in ArchAgents:
\[CRL=(\mathcal{G}_{CRL},Know_{CRL},\Phi_{CRL})\]

From the knowledge perspective, CRL introduces a structured and differentiated treatment of knowledge. Persistent knowledge is no longer collapsed into a single carrier but is distributed across two conceptually distinct components:
\begin{itemize}
\item $\Theta_\pi^k$, encoding policy-related parametric knowledge,
\item $\Theta_{CS}^k$, encoding causal knowledge about the environment.
\end{itemize}

This clear separation enables the architecture to explicitly represent and reuse causal information across interactions (along with the $Do$ generator), rather than embedding it implicitly into parameters. The causal model $\Theta_{CS}^k$ (along with the $Do$ generator) works as an intermediary informational structure that shapes both learning and control.
The object $\Theta_{CS}^k$ within the knowledge layer $Know_{CRL}$, unlike $\Theta_\pi^k$, is not treated as a purely atomic carrier. Architecturally, it is required to support transformations corresponding to:
\begin{itemize}
\item updating causal structure from experience,
\item conditioning policy updates on causal information,
\item mediating action selection through causal reasoning and do-intervention.
\end{itemize}

Learning in CRL thus decomposes into two coordinated endomorphic processes: $PolicyUpd$ and $CausalUpd$.
Additionally, updating current knowledge is also possible with $CausalIntervention$.
As in the RL case, this does not impose \textbf{probabilistic semantics, causal discovery algorithms, or identifiability guarantees}. These aspects are intentionally left to concrete algorithmic instantiations, while the architecture specifies only the admissible information flows and structural roles.
In particular, the architecture exhibits two interacting feedback loops: a causal learning loop over $\Theta_{CS}^s$ and a policy learning loop over $\Theta_\pi^s$, coupled through shared access to experience and causal structure.
This interpretation makes explicit a key architectural distinction with respect to classical RL: causal knowledge is no longer implicit in parameters but is represented, updated, and reused as a first-class informational component of the agent architecture.

\paragraph{CRL Architecture Properties}
Structural properties:
\begin{itemize}
    \item \textbf{Multiple coupled feedback loops.}  
    The architecture admits multiple, non-collapsible feedback loops, corresponding to distinct update cycles for the policy parameters $\Theta_\pi$ and for the causal model parameters $\Theta_{CS}$. These loops are structurally distinct but coupled through shared experience and interventions on the causal model.
    
    \item \textbf{Explicit separation of decision and learning regimes.}  
    The wiring of $Syn_{CRL}$ enforces a structural distinction between the model used for action selection and the intervened model under which learning takes place.
    
    \item \textbf{Typed causal mediation.}  
    Actions do not directly influence learning modules but act through explicit causal mediation morphisms (e.g.\ $\mathsf{Do}$), making intervention a first-class structural component of the architecture.
    
\end{itemize}
Informational properties:
\begin{itemize}
    \item \textbf{Partial decomposition of persistent information.}  
    Persistent information is no longer encoded in a single carrier: the architecture separates control knowledge ($\Theta_\pi$) from causal world knowledge ($\Theta_{CS}$), thus, enabling differentiated update and use.
    
    \item \textbf{Causal conditioning of learning.}  
    Learning updates are conditioned on the intervened causal models induced by actions, rather than on purely observational information flows.
    
    \item \textbf{Limited knowledge modularity.}  
    While the architecture distinguishes between policy knowledge and causal knowledge, each remains internally monolithic. The architecture provides no mechanisms for decomposing, isolating, or recombining subcomponents of $\Theta_{CS}$ or $\Theta_\pi$.
    
    \item \textbf{Absence of hierarchical or reusable knowledge units.}  
    Knowledge is persistent and structured by role, but not by scale or reuse. That is, learned causal or control structures cannot be encapsulated, composed, or redeployed as independent informational units.
\end{itemize}

While Causal Reinforcement Learning constitutes a significant architectural improvement over the RL Architecture, it still remains insufficient to account for a wide range of cognitive capabilities that are central for \textbf{continual learning}. These include \textbf{latent variable discovery, macroaction learning, concept learning, goal adaptation, and knowledge reutilization} across different tasks and contexts.
Several architectural paradigms partially address these challenges, such as \textbf{Multi-Model Reinforcement Learning (MMRL)} and \textbf{Hierarchical Reinforcement Learning (HRL)}, 
which introduce modularity or hierarchy along specific dimensions. MMRL promotes the modularization of knowledge through the use of multiple specialized models, while HRL introduces hierarchical structure primarily in the space of actions, enabling temporal abstraction and multi-level control. However, these approaches typically focus on isolated aspects of cognition and do not provide a general architectural account of knowledge organization.

Thus, the final Case Study introduces the \textbf{Schema-Based Learning (SBL)} architecture, which aims to provide a unified architectural framework for \textbf{modular, compositional, and reusable knowledge}. In SBL, schemas and workflows play a central role in structuring knowledge and governing its interaction, thus, offering a principled foundation for continual learning and architectural knowledge scalability. Moreover, SBL enables the agent to progressively
develop the diverse cognitive capabilities discussed above within a
single, coherent architectural framework.

\subsection{From CRL to SBL: Relaxing Architectural Design Principles Assumptions}

The transition from Causal Reinforcement Learning (CRL) to the full Schema-Based Learning (SBL) architecture is not presented as an abrupt architectural replacement, but rather as a sequence of step-wise conservative extensions of the underlying syntactic structure.
At each step, a specific architectural constraint of CRL is relaxed, while the rest of the architecture is left unchanged.
This allows the reader to track precisely which assumptions are removed and how SBL emerges as their joint generalization.
Throughout this subsection, we focus exclusively on the syntactic architecture and its string diagram representation. The full formalization is deferred to the SBL case study.

\paragraph{Baseline: Canonical CRL.}
We start from the canonical CRL architecture introduced in the previous section (Figure~\ref{fig:CRL string diagram}).
In this setting, the agent is characterized by:
\begin{enumerate}
    \item atomic observation and decision interfaces,
    \item a single global causal predictive model and a single global policy model,
    \item a single monolithic update mechanism for both models, and 
    \item a unique feedback loop through which all learning occurs.
\end{enumerate}
This configuration fixes the set of architectural constraints that will be progressively relaxed in the following steps.

\subsubsection{Step 1: Factorization of Interfaces.}
The first relaxation concerns the structure of the agent-environment interface.
Instead of assuming atomic observation and decision spaces, we allow both interfaces to be factorized into multiple components
(Figure~\ref{fig:factored_crl}). Importantly, this modification does not introduce new syntactic processes: the learning loop, update mechanism, and model structure remain unchanged.
However, from a knowledge perspective, this step enriches the information and knowledge management of the architecture, decreasing the risk of suffering the \textbf{curse of dimensionality} while preserving its overall topology. Thus, this factorization is a necessary prerequisite before all the subsequent extensions, because otherwise the problem of dimensionality would be carried over to SBL.
This step only causes changes in the syntactic part, maintaining the knowledge management from CRL
\paragraph{Syntactic layer}
\begin{itemize}
    \item $STypes: \{O_i\},\{D_j\}, E, \Theta_\pi^s, \Theta_{CS}^s$, where $\{O_i\},\{D_j\}$ are the factorized set of observation and decision types that "enrich" the types $S$ and $A$ with a new constraint, obliging these types to hold factorized representations.
    \item In the case of $SGen$, the same generators are kept  (updating the corresponding input and output types), namely, $\mathsf{Policy}, \mathsf{EnvInteraction}, \mathsf{Do}, \mathsf{PolicyUpdate} $ and $\mathsf{CausalUpdate}$.
    \item $SEq$ remains empty.
\end{itemize}


\subsubsection{Step 2: Typed Multi-Model Architecture.}
The second relaxation removes the assumption of an unique global internal model of RL, but in a more general way than CRL has done.
Instead, the agent is allowed to manage a wide collection of different internal models, each operating over (possibly different) subinterfaces of the factorized observation and decision spaces and learned or used for different goals (Figure~\ref{fig:mmrl_string_diagram}) 

. At this stage, these models coexist in parallel and are not yet organized by any higher-level coordination mechanism. Syntactically, this corresponds to replacing a single predictive component with a family of typed components, while still preserving the single learning loop. This step introduces the structural basis for schemas, although their full role will be later formalized.
This step causes changes in the syntactic and knowledge parts, 
\paragraph{Syntactic layer}
\begin{itemize}
    \item $STypes: \{O_i\},\{D_j\}, E, \{\Theta_\pi^s\}, \{\Theta_{CS}^s\}, \{\theta_{CS}^s\}, \{\theta_{\pi}^s\}$, where $\{\Theta_{CS}^s\},\{\Theta_\pi^s\}$ are the types representing the global set of causal and policy models respectively, that is, all the models that the agent stores, and $\{\theta_{CS}^s\}, \{\theta_{\pi}^s\}$ are the local set of models selected in each iteration. This enables the agent to work with multiple models instead of relying on unique models.
    \item For $SGen$, we maintain the old generators (updating the corresponding inputs and outputs types), and two new generators are added for managing the selection and aggregation of new models into the global set: $ \mathsf{SelectModels}, \mathsf{AggModels}, \mathsf{Policy}, \mathsf{EnvInteraction}, \mathsf{Do}, \mathsf{PolicyUpdate} $ and $\mathsf{CausalUpdate}$
    \item $SEq$ remains empty.
\end{itemize}

\paragraph{Knowledge layer}
\begin{itemize}
    \item $KTypes: \{\Theta_\pi^k\}, \{\Theta_{CS}^k\}, \{\theta_{CS}^k\}, \{\theta_{\pi}^k\}$. We replace the unique knowledge carrier units for four knowledge carriers, two global carriers for causal and policy knowledge, and another two local carriers representing ...
    \item For $KGen$  we maintain the same generators, updating the domains, and also add the two corresponding knowledge generators for managing the :$\mathsf{SelectModels}, \mathsf{AggModels},$
    \item $KEq$ remains empty.
\end{itemize}
In this step the profunctor has some changes, since some changes are made in the knowledge types and in the syntactic types related with knowledge
\begin{table}[h]
\centering
\renewcommand{\arraystretch}{1.25}
\setlength{\tabcolsep}{18pt}

\begin{tabular}{|c|c|c|c|c|c|c|c|}
\hline
\textbf{Types} & $\{O_i\}$ & $\{D_j\}$ & $E$ & $\{\Theta^s_\pi\}$ & $\{\Theta^s_{CS}\}$ & $\{\theta^s_{\pi}\}$ & $\{\theta^s_{CS}\}$ \\
\hline
$\{\Theta^k_\pi\}$ & \xmark & \xmark & \xmark & \cmark & \xmark & \xmark & \xmark\\
\hline
$\{\Theta^k_{CS}\}$ & \xmark & \xmark & \xmark & \xmark & \cmark & \xmark & \xmark \\
\hline
$\{\theta^k_{\pi}\}$ & \xmark & \xmark & \xmark & \xmark & \xmark & \cmark & \xmark \\
\hline
$\{\theta^k_{CS}\}$ & \xmark & \xmark & \xmark & \xmark & \xmark & \xmark & \cmark \\
\hline

\end{tabular}
\caption{Visualization as a table of the support of the relational profunctor}
\end{table}

\subsubsection{Step 3: Cognitive Modules}
Next, we relax the assumption that all the internal processing within the agent can be subsumed under a single, homogeneous learning or update mechanism. Instead, the architecture is generalized to allow for multiple heterogeneous internal processes, which we collectively refer to as cognitive modules (Figure~\ref{fig:cog_mod_string_diagram}).
Cognitive modules abstract over a wide range of internal capabilities, including but not limited to decision-making, causal learning, latent variable learning, macroaction learning, concept learning, goal adaptation, memory management, etc. At this stage, no commitment is made regarding the specific function, learning algorithm or semantics of each module. The key architectural change is purely syntactic: internal processing is no longer assumed to be centralized, but distributed across multiple, potentially specialized components, that get as inputs some schemas, signals or memory systems, and produce output signals. These output signals may in turn update the schemas or the memory, or perform a decision.
Moreover, apart from operating on specific parts, the modules may be selectively activated depending on the agent's current context.
This step causes some changes both in the syntactic and the knowledge parts. The goal is to make the architecture more expressive. 
\paragraph{Syntactic layer}
\begin{itemize}
    \item $STypes: \{O_i\},\{D_j\}, E, \{\Theta^s\}, \{\theta^s\}, \mathcal{K}$, where $\mathcal{K}$ are the cognitive module identifiers and $\{\Theta^s\},\{\theta^s\}$ are the types representing the global and local sets of models, respectively. Yet, in this case, the types of models are not restricted only to policy or causal models, but they can be any type of model. This enables the agent to work with multiple types of models instead of relying on certain types of models. This allows to cover a wide variety of cognitive capabilities. A new type of selected cognitive modules set is added. This set decides which flows will be executed subsequently in the same agent iteration. 
    \item For $SGen$, we maintain the old generators $\mathsf{EnvInteraction}$ and $\mathsf{AggModels}$ but some changes are required for other generators, namely:
    \begin{itemize}
        \item $CogModActivation: \{\Theta^s\} \otimes \{O_i\} \to \{\theta^s\} \otimes \mathcal{K}$. This generator not only selects the models for the current iteration, but also selects the cognitive modules $\mathcal{K}$ that will be activated.
        \item $CogModExecution: \{\theta^s\} \otimes \mathcal{K} \to \{\theta^s\} \otimes \{D_j\}$. This generator encompasses the execution of the corresponding flows in the iteration. Since we aim to generalize in this step, we do not design the specific cognitive module flows that this generator represents. For providing some examples of cognitive modules, the flow for deciding the policy that we have seen in previous architectures will be included in the cognitive module of decision-making. 
        \item $UpdateModels:\{\theta^s\} \otimes E \to \{\theta^s\}$. This generator updates the local models selected in the iteration (after being modified by the cognitive modules) and the experience observed
    \end{itemize}
    \item $SEq$ remains empty.
\end{itemize}

\paragraph{Knowledge layer}
\begin{itemize}
    \item $KTypes: \{\Theta^k\}, \{\theta^k\}, \theta^k_1,\dots,\theta^k_n$. We replace the typed knowledge carrier units by a global and local set of carrier units and the different types of knowledge that the agent will lead with. 
    \item For $KGen$  we maintain the same generators $\mathsf{SelectModels}, \mathsf{AggModels}$, and we add all the possible operators that can transform or update the knowledge types that represent the types of knowledge units, we denote it by:
    \begin{itemize}
        \item $Upd_i:\theta_i^k \to \theta_i^k$ the generators responsible for updating the corresponding knowledge unit types.
        \item $Transform:\theta_*^k \to \theta_*^k$ the generators responsible for transforming one type of knowledge unit into another one (we use the notation $\theta_*$ to generalize the type of knowledge).
    \end{itemize}
    \item $KEq$ remains empty depending on the algebraic rules that may be added to the generators $Upd$ or $Transform$.
\end{itemize}
In this step the profunctor has some changes, since some changes are made in the knowledge types and in the syntactic types related with knowledge

\begin{table}[h]
\centering
\renewcommand{\arraystretch}{1.25}
\setlength{\tabcolsep}{18pt}

\begin{tabular}{|c|c|c|c|c|}
\hline
\textbf{Types} & $\{O_i\},\{D_j\},E$ &$\mathcal{K}$ &$\{\Theta^s\}$ & $\{\theta^s\}$ \\
\hline
$\{\Theta^k\}$ & \xmark & \xmark & \cmark & \xmark  \\
\hline
$\{\theta^k\}$ & \xmark & \xmark & \xmark & \cmark    \\
\hline
$\theta^k_1$ & \xmark & \xmark & \xmark & \cmark   \\
\hline
$\cdots$ & \xmark & \xmark & \xmark & \cmark  \\
\hline
$\theta^k_{n}$ & \xmark & \xmark & \xmark & \cmark  \\
\hline
\end{tabular}
\caption{Visualization as a table of the support of the relational profunctor}
\end{table}
\subsubsection{Step 4: Temporal Decoupling and Memory.}
The next architectural relaxation concerns the role of the past experiences in learning and internal processing. Rather than assuming that all learning and adaptation must occur strictly online, at the moment of interaction with the environment, we allow the agent to retain and reuse past experiences through an explicit memory structure (Figure~\ref{fig:memory string diagram}). This memory structure can also potentially include several different types of memories, such as working memory, and emotional memory among others.
This modification is inspired by offline and batch reinforcement learning settings, in which learning processes are driven by queries to stored experience rather than by focusing solely on immediate environmental feedback. The introduction of memory enables internal models (e.g. schemas) and cognitive modules to access, revisit, and reorganize past interactions whenever required, supporting learning regimes that are not tied to a single temporal scale or update schedule.
Importantly, this extension does not replace online learning (included in the short-term memory). Instead, the architecture is generalized to accommodate both online and memory-based learning within the same syntactic framework. Some cognitive modules may rely primarily on immediate experience, while others may operate by querying long-term memory to perform retrospective analysis, or may rely on memory related with previous cognitive process executions.
At the syntactic level, memory appears as a persistent type object that mediates access to stored experience. A loop that updates memory also appears, yet the overall feedback structure of the architecture remains unchanged.
Thus, the introduction of memory indirectly relaxes the constraint of strict temporal synchrony between learning and decision-making.

This step causes changes in the syntactic and knowledge parts. 
\paragraph{Syntactic level}
\begin{itemize}
    \item $STypes: \{O_i\},\{D_j\}, E, \{\Theta^s\}, \{\theta^s\}, \mathcal{K}, \mathcal{M}^s$, where $\mathcal{M}^s$ represent the type that will hold the memory 
    \item For $SGen$, we maintain and update the old generators and we also add a new one for managing the aggregation of memories: $AggMem:\mathcal{M^s}\otimes\{\theta^s\} \to \{\theta^s\}$
    \item $SEq$ remains empty.
\end{itemize}

In the knowledge layer we would have the following changes
\paragraph{Knowledge level}
\begin{itemize}
    \item $KTypes: \{\Theta^k\}, \{\theta^k\}, \theta_1^k,\dots,\theta_n^k, \mathcal{M}^k$. We add the new unit of knowledge that handle the memory
    \item For $KGen$ we maintain the same generators, adding the generator that will handle the aggregation of memories $AggMem:\mathcal{M}^k \to \mathcal{M}^k$
    \item $KEq$ remains empty depending on the algebraic rules that want to be added to the generators $Upd$ or $Transform$.
\end{itemize}
In this step the profunctor has some changes, since some changes are made in the knowledge types and in the syntactic types related with knowledge.

\begin{table}[h]
\centering
\renewcommand{\arraystretch}{1.25}
\setlength{\tabcolsep}{18pt}

\begin{tabular}{|c|c|c|c|c|}
\hline
\textbf{Types} & $\{O_i\},\{D_j\},E,\mathcal{K}$ &$\mathcal{M}^s$ &$\{\Theta^s\}$ & $\{\theta^s\}$ \\
\hline
$\{\Theta^k\}$ & \xmark & \xmark & \cmark & \xmark  \\
\hline
$\{\theta^k\}$ & \xmark & \xmark & \xmark & \cmark    \\
\hline
$\theta^k_1$ & \xmark & \xmark & \xmark & \cmark   \\
\hline
$\cdots$ & \xmark & \xmark & \xmark & \cmark  \\
\hline
$\theta^k_{n}$ & \xmark & \xmark & \xmark & \cmark  \\
\hline
$\mathcal{M}^k$ & \xmark & \cmark & \xmark & \xmark  \\
\hline
\end{tabular}
\caption{Visualization as a table of the support of the relational profunctor}
\end{table}
\subsubsection{Step 5: Body-Mind Mediation}
The final relaxation concerns the interpretation of the agent's interfaces with the environment.
Up to this point, we have implicitly assumed that the environment directly provides observations compatible with the factorized internal interfaces. We now remove this assumption by introducing a mediating layer, referred
to as the Body, which transforms raw environmental signals into structured observations suitable for the internal architecture part
(Figure~\ref{fig:body-mind string diagram}). The Mind retains the previously introduced syntactic structure, including schemas, cognitive modules, and memory. This modification does not affect the internal learning dynamics, but clarifies the separation between the semantic interaction with the environment and the syntactic organization of internal processes.

This step only causes changes in the syntactic part. 
\paragraph{Syntactic level}
\begin{itemize}
    \item $STypes: S,A,\{O_i\},\{D_j\}, E, \{\Theta^s\}, \{\theta^s\}, \mathcal{K}, \mathcal{M}$. In this case we reintroduced the types of the state and action, maintaining the types of factorized observations and decisiones.
    
    \item For $SGen$, we maintain the old generators and add two new generators for handling the factorization of states and actions into observation and decisions
    \begin{itemize}
        \item $FactorState: S \to \{O_i\}$.
        \item $FactorAction: \{D_j\} \to A$.
    \end{itemize}
    \item $SEq$ remains empty.
\end{itemize}

\begin{figure}[!ht]
    \centering
    \begin{subfigure}[b]{0.49\textwidth}
        \includegraphics[width=\linewidth]{figures/string_diagrams/CRL_string_diagram.png}
        \caption{Step 0: CRL Baseline}
        \label{fig:factored_crl}    
    \end{subfigure}
    \hfill
    \begin{subfigure}[b]{0.49\textwidth}
        \includegraphics[width=1.1\linewidth]{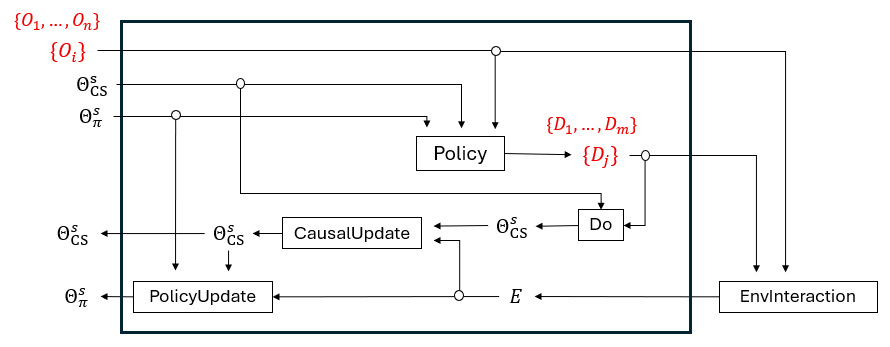}
        \caption{Step 1: Factored representation addition}
        \label{fig:factored_crl}    
    \end{subfigure}
    \hfill
    \begin{subfigure}[b]{0.49\textwidth}
        \includegraphics[width=\linewidth]{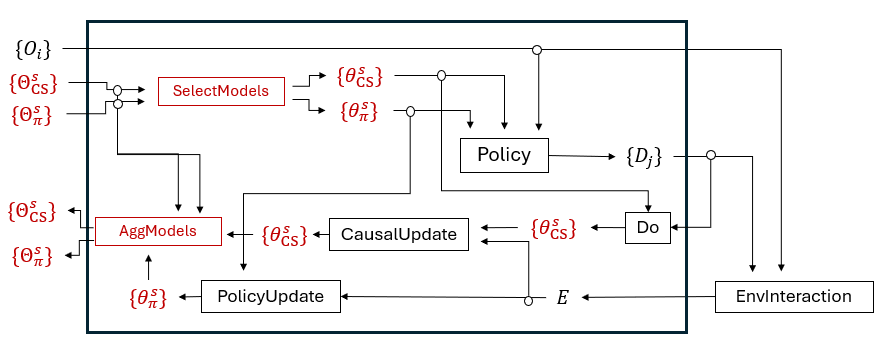}
        \caption{Step 2: Multiple Models addition}
        \label{fig:mmrl_string_diagram}    
    \end{subfigure}
    \hfill
    \begin{subfigure}[b]{0.49\textwidth}
        \includegraphics[width=1.1\linewidth]{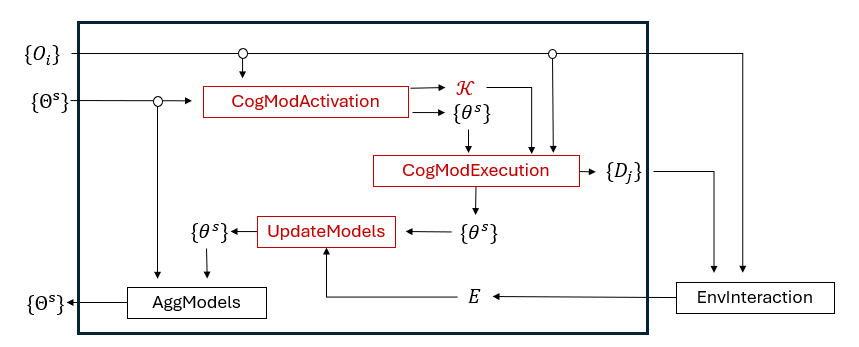}
        \caption{Step 3: Cognitive modules generalization addition}
        \label{fig:cog_mod_string_diagram}    
    \end{subfigure}
    \hfill
    \begin{subfigure}[b]{0.49\textwidth}
        \includegraphics[width=\linewidth]{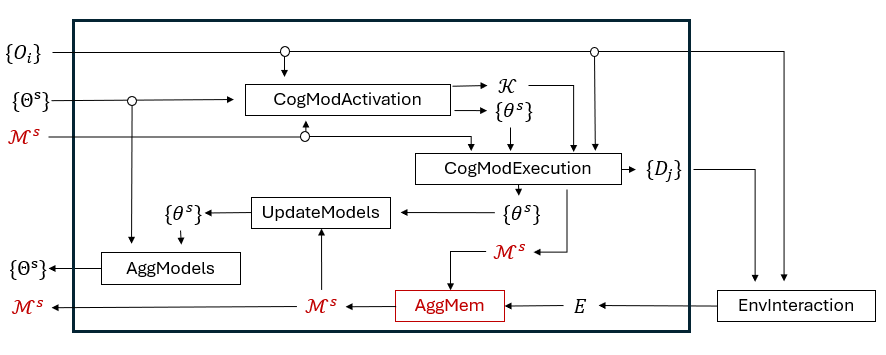}
        \caption{Step 4: Memory addition}
        \label{fig:memory string diagram}    
    \end{subfigure}
    \hfill
    \begin{subfigure}[b]{0.49\textwidth}
        \includegraphics[width=\linewidth]{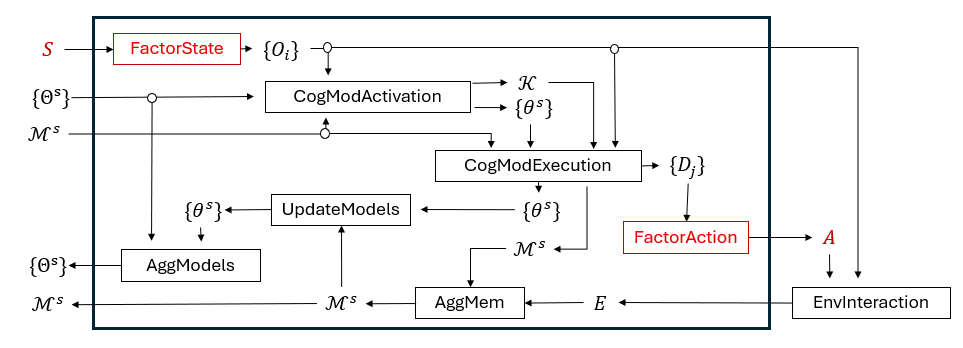}
        \caption{Step 5: Embodiment addition}
        \label{fig:body-mind string diagram}    
    \end{subfigure}
    \label{fig:Step_CRL_to_SBL}
    \caption{Figures of consecutive change steps over the CRL architecture. The changes done in each step compared with the previous one are colored in red.}
\end{figure}

\paragraph{Towards SBL}
The resulting architecture corresponds to the prototype of the SBL Architecture case study developed in the next section. Thus, SBL emerges not as a departure from CRL, but as the minimal architecture obtained by systematically relaxing CRL’s and RL's assumptions about interface structure, model uniqueness, update atomicity, temporal synchrony, and environmental compatibility.

\subsection{Case Study III: Schema-Based Learning Architecture}

Schema-Based Learning (SBL) instantiates a strongly modular cognitive architecture, characterized by an explicit separation between body-level interaction, mind-level processing, and structured internal knowledge dynamics.
Unlike RL and CRL, SBL features multiple interacting cognitive modules, localized knowledge updates, and heterogeneous schema types.
Before specifying the architectural presentation, we briefly clarify the notion of schema within our categorical framework.
Informally, a schema is a structured, reusable unit of knowledge that encapsulates a pattern of interaction between an input and an output. Unlike the monolithic parameter carriers of RL or CRL, schemas are modular, typed knowledge components with their own (intrinsic) parameters. Schemas  may also be instantiated, composed, activated, and updated independently.
Formally, at the architectural level, a schema is represented as an object in the knowledge category $Know_{SBL}$, together with a family of generators describing admissible schema transformations (creation, refinement, combination, contextualization, etc.). Thus, schemas are not primitive algorithms but algebraic knowledge carriers equipped with structured workflows.


\subsubsection{SBL syntactic level.}

The hypergraph category $Syn_{SBL}$ is freely generated by the following
types symbols:
\begin{itemize}
    \item $S$ raw sensory data (body-level sensors),
    \item $A$ raw effector commands (body-level actuators),
    \item $\{O_i\}$ structured factorized observations in mind,
    \item $\{D_j\}$ structured factorized decisions in mind,
    \item $E$ the experience type,
    
    \item $\{\Theta^s\}$: the global set of models carriers (global set of schemas),
    \item $\{\theta^s\}$ the local models carriers,
    \item $\mathcal{M}^s$ global memory state, encompassing multiple memory systems,
    \item $\mathcal{K}$ cognitive module identifiers,
    
\end{itemize}
All tensor expressions over these types are available via the symmetric
monoidal structure.
The primitive architectural generators of $Syn_{SBL}$ include:
\[
\begin{aligned}
&\mathsf{PerceptInst}
    : S \longrightarrow \{O_i\}, \\[0.3em]
&\mathsf{MotorInst}
    : \{D_j\} \longrightarrow A, \\[0.6em]
&\mathsf{CogModActivate}
    : \{O_i\} \otimes \{\Theta^s\} \otimes \mathcal{M}^s
    \longrightarrow \{\theta^s\} \otimes \mathcal{K}, \\[0.6em]
&\mathsf{CogModExec}
    : \{\theta^s\} \otimes \mathcal{K} \otimes \{O_j\} \otimes \mathcal{M}^s
    \longrightarrow \{D_j\} \otimes \{\theta^s\} \otimes \mathcal{M}^s, \\[0.6em]
&\mathsf{UpdateSchemas}
    : \{\theta^s\} \otimes \mathcal{M}^s
    \longrightarrow \{\theta^s\}, \\[0.3em]
&\mathsf{AggMem}
    : \mathcal{M}^s \otimes E
    \longrightarrow \mathcal{M}^s, \\
&\mathsf{AggSchemas}
    : \{\theta^s\} \otimes \{\Theta^s\}
    \longrightarrow \{\Theta^s\},\\
&\mathsf{EnvInteraction}
    : S \otimes A
    \longrightarrow E.
\end{aligned}\]

These generators represent abstract architectural roles, that is:
\begin{itemize}
    \item the transformation between body-level signals and structured
    mental representations via perceptive and motor schemas,
    \item the activation and routing of cognitive modules based on current
    observations, memory state, and internal models,
    \item the execution of cognitive modules, producing decisions as well as
    localized effects on schemas and memory,
    \item the explicit and localized integration of schema and memory updates
    into the global memory state and the internal models.
\end{itemize}

Importantly, learning and other cognitive operations are not represented as a primitive architectural operation. Instead, they are realized internally within specific cognitive modules during $\mathsf{CogModExec}$. They become observable at the
architectural level only through the induced schema and memory updates. This design decouples learning from immediate action execution and allows cognitive modules to operate asynchronously and in parallel, driven by memory contents rather than by online interaction alone.
We remind the reader that knowledge management is not a primitive architectural action but is realized through induced knowledge workflows associated with $CogModExec$.
No additional relations are imposed beyond those required by the axioms of hypergraph categories. In particular, $Syn_{SBL}$ does not prescribe any specific learning algorithm, optimization objective, or control policy.  The resulting architectural pattern $\mathcal{G}_{SBL}$ is fully determined by the string diagram in Figure \ref{fig:SBL string diagram}.

\begin{figure}[!ht]
    \centering
    \includegraphics[width=0.75\linewidth]{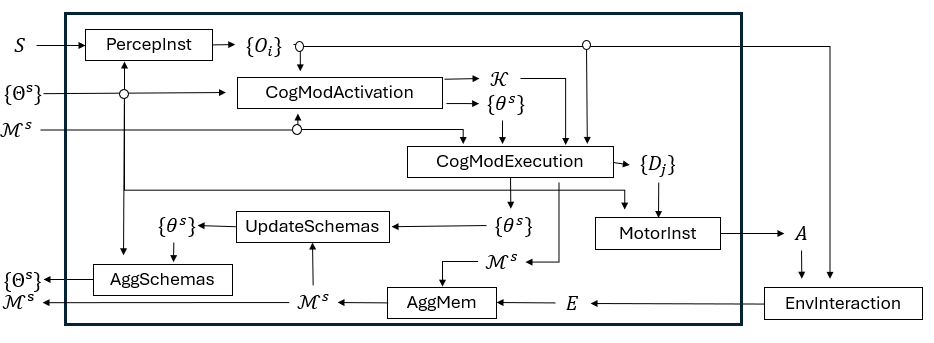}
    \caption{SBL string diagram}
    \label{fig:SBL string diagram}
\end{figure}

\subsubsection{SBL Knowledge level.}

The knowledge architecture $Know_{SBL}$ is freely generated by the knowledge presentation $K_{SBL}=(KTypes_{SBL},KGen_{SBL},KEq_{SBL})$ where:
\begin{itemize}
    \item $KTypes_{SBL} = \{ \Sigma_{Perceptual}, \Sigma_{Motor}, \Sigma_{Predictive}, \Sigma_{Reward}, \Sigma_{Abstract}, \{\Theta^k\}, \mathcal{M}^k \}$. The knowledge architecture $Know_{SBL}$ comprises a heterogeneous family of schemas together with a global memory $\mathcal{M}^k$ and a global carrier of schemas $\{\Theta^k\}$.
    \item Knowledge evolves through a small set of \emph{schema generators}, which define the possible primitive transformations over schemas. Intuitively, schemas can be created from experience, refined or updated through learning, combined into more complex structures, encapsulated into higher–level abstractions, and contextualized depending on the active situation. Additional operators regulate schema selection and the aggregation of information into global memory. These generators can be formalized as typed morphisms describing elementary transformations over schemas and memory states. Therefore, a knowledge workflow in SBL is a morphism in $Know_{SBL}$ obtained by composing the following generators in $KGen_{SBL}$. We denote by $\Sigma_*$ any type of schemas: 
    \begin{itemize}
        \item $SchemaCreate:\mathcal{M}^k \to \Sigma_* \otimes \mathcal{M}^k$
        \item $SchemaDelete:\Sigma_* \otimes \{\Theta^k\} \to \{\Theta^k\}$
        \item $SchemaCombine:\Sigma_* \otimes \Sigma_* \to \Sigma_*$
        \item $SchemaRefine:\Sigma_* \to \Sigma_*$
        \item $SchemaEncap: \Sigma_* \otimes \Sigma_* \to \Sigma_*$
        \item $SchemaCtx:\Sigma_* \to \Sigma_*$
        \item $SchemaUpd:\Sigma_* \to \Sigma_*$
        \item $SchemaSelect:\{\Theta^k\} \to \Sigma$
        \item $AggMem:\mathcal{M}^k \otimes \{\Theta^k\} \to \mathcal{M}^k$
        \item $AggSchemas:\{\Theta^k\} \otimes \Sigma_*  \to \{\Theta^k\}$
    \end{itemize}
    \item $KEq_{SBL}$ is empty by now.
\end{itemize}

\paragraph{Relational Interface for SBL.}
The interaction between architectural workflows and internal knowledge resources in SBL is specified by the relational interface profunctor
\[ \Phi_{SBL} : \mathcal{G}_{SBL}^{\mathrm{op}} \nrightarrow \mathsf{Know}_{SBL}\]

At the object level, this profunctor has non-empty support on the pairs
\[\begin{gathered}
    (\{\theta^s\},\Sigma_{Perceptual}), \quad (\{\theta^s\},\Sigma_{Motor}), \quad (\{\theta^s\},\Sigma_{Pred}), \quad (\{\theta^s\},\Sigma_{Reward}), \\
    (\{\theta^s\},\Sigma_{Abstract}), \quad (\mathcal{M}^s,\mathcal{M}^k), \quad (\{\Theta^s\},\{\Theta^k\})
\end{gathered}\] 
indicating that syntactic diagrams may access and transform both
schema-level and memory-level knowledge resources.
In particular:
\begin{itemize}
    \item cognitive module activation and execution workflows are classified by elements of $\Phi_{SBL}(d,\Sigma)$ whenever they require access to, selection of, or composition over active schemas;

    \item syntactic diagrams that do not activate a schema are not related by the profunctor, ensuring isolation of inactive knowledge.
\end{itemize}

\begin{table}[!ht]
\centering
\renewcommand{\arraystretch}{1.25}
\setlength{\tabcolsep}{18pt}
\begin{tabular}{|c|c|c|c|c|c|c|}
\hline
\textbf{Types} & $S$,$A$,$\{O_i\}$,$\{D_j\}$,$E$ & $\mathcal{K}$ & $\{\theta^s\}$ & $\mathcal{M}^s$ & $\{\Theta^s\}$ \\
\hline
$\Sigma_{Perceptual}$ & \xmark & \xmark & \cmark & \xmark & \xmark  \\
\hline
$\Sigma_{Motor}$ & \xmark & \xmark & \cmark & \xmark & \xmark\\
\hline
$\Sigma_{Pred}$ & \xmark & \xmark & \cmark & \xmark & \xmark\\
\hline
$\Sigma_{Reward}$ & \xmark & \xmark & \cmark & \xmark & \xmark\\
\hline
$\Sigma_{Abstract}$ & \xmark & \xmark & \cmark & \xmark & \xmark\\
\hline
$\mathcal{M}^k$ & \xmark & \xmark & \xmark & \cmark & \xmark\\
\hline
$\{\Theta^k\}$ & \xmark & \xmark & \xmark & \xmark & \cmark \\
\hline

\end{tabular}
\caption{Visualization as a table of the support of the SBL relational interface profunctor}
\end{table}

This profunctorial structure makes explicit that SBL supports modular, localized, and compositional knowledge access patterns, in contrast with the centralized knowledge levels of RL and CRL.

These schema families correspond to the main functional components of the architecture. Namely, 
\begin{itemize}
    \item Perceptual schemas $\Sigma_{Perceptual}$ transform internal or environmental states into observations available to the agent;
    \item Motor schemas $\Sigma_{Motor}$ translate decisions into executable actions in the environment;
    \item Predictive schemas $\Sigma_{Pred}$ learn relations between observations, rewards, and decisions in order to produce predictions about future observations, actions, or rewards;
    \item Reward schemas $\Sigma_{Reward}$ evaluate situations by assigning a real-valued utility signal from observations and decisions, and
    \item Abstract schemas $\Sigma_{Abstract}$ represent higher-level compositional structures over the previous schema types through operations such as combination, encapsulation, or refinement.
\end{itemize}
\subsubsection{SBL Constraints}
As in the previous case studies, this presentation captures the compositional and epistemic organization of SBL, but not yet the full family of architecture-specific admissibility constraints that would characterize which realizations genuinely count as schema-based in the stronger sense. In particular, one expects additional restrictions governing the admissible ontology of schema types, the legality of schema transformations, the modularity and locality of schema updates, and the compatibility conditions between cognitive modules, memory, and schema composition. A first detailed study of this constraint layer is deferred to \cite{RiscosCATSBL}.

\subsubsection{SBL as object of ArchAgents}

Finally the SBL architecture would be described as the following object in ArchAgents:
\[SBL=(\mathcal{G}_{SBL},Know_{SBL},\mathcal{T}_{SBL})\]

Unlike Reinforcement Learning (RL) and Causal Reinforcement Learning (CRL), where persistent knowledge is centralized into a single parametric carrier, the Schema-Based Learning (SBL) architecture exhibits a fundamentally modular treatment of information. Persistent knowledge is not collapsed into a global state but instead resides in a collection of schemas, each constituting an independent informational unit. This collection is not fixed: the agent begins with a set of primitive schemas, and new schemas may be created, transformed, composed, or discarded as the agent evolves.
Schemas are typed functional entities that map between informational spaces and are classified into perceptual schemas, motor schemas, predictive schemas, reward schemas, and abstract schemas constructed compositionally from the former. Although schemas may operate over similar spaces, they remain informationally distinct unless explicitly identified as identical.
Cognitive processing in SBL is organized through cognitive modules, each governed by a fixed workflow. Workflows specify admissible patterns of schema activation, composition, and update during cognitive processes such as perception, decision making, prediction, or learning. While workflows are fixed at the architectural level, they control how schemas are coordinated and transformed during execution. Knowledge management in ${Know}_{SBL}$ is based on these principles:
\begin{itemize}
    \item \textbf{Knowledge Modularity}: schemas are distinct, identifiable units of knowledge.
    \item \textbf{Knowledge Compositionality}: schemas may be composed to form higher-level or abstract schemas.
    \item \textbf{Locality of learning}: learning updates apply only to schemas activated within a workflow.
    \item \textbf{Knowledge Partial isolation}: inactive schemas are not affected by unrelated learning processes.
\end{itemize}

These constraints are architectural rather than algorithmic and hold independently of any particular learning rule or representational choice.

\subsubsection{Properties}

\paragraph{Structural Properties of SBL}

The SBL architecture is characterized by a strongly modular and compositional syntactic organization. Unlike monolithic architectures, the decision-making process is explicitly decomposed into interacting perceptual, predictive, decision, and learning components.

Examples of structural properties include:

\begin{itemize}

\item \textbf{Architectural modularity.}
The perception-action process is decomposed into multiple interacting modules rather than a single global computation. Different schemas may participate in different workflows.

\item \textbf{Compositional workflows.}
Behavior emerges from the composition of schema instantiations and knowledge transformations. Complex behaviors can therefore be constructed from simpler reusable components.

\item \textbf{Workflow flexibility.}
The architecture admits multiple admissible knowledge workflows rather than enforcing a single fixed reasoning pipeline. Different situations may activate different subsets of schemas.

\item \textbf{Local processing.}
Information processing can be restricted to subsets of the architecture. Not all modules need to participate in every decision cycle.

\item \textbf{Hierarchical organization.}
Schemas may be composed to form higher-level schemas, enabling the emergence of hierarchical representations and multi-scale reasoning processes.

\item \textbf{Incremental structural growth.}
The architecture permits the creation, refinement, differentiation, and specialization of schemas during the lifetime of the agent, allowing the structural organization itself to evolve.

\item \textbf{Reusability of components.}
Individual schemas can participate in multiple workflows and contexts without requiring duplication of the underlying knowledge structures.

\item \textbf{Structural extensibility.}
New schemas and schema relations can be incorporated without requiring a complete redesign of the existing architecture.

\end{itemize}

\paragraph{Informational Properties of SBL}

The category ${Know}_{\mathrm{SBL}}$ satisfies several distinctive informational properties arising from its schema-based organization of knowledge.

\begin{itemize}

\item \textbf{Knowledge modularity.}
Knowledge is partitioned into schemas that constitute distinct informational units. Each schema encapsulates a localized portion of the agent's knowledge.

\item \textbf{Modularity by construction.}
Schemas remain explicit entities within the architecture and can be composed, reused, refined, or specialized without being collapsed into a single global representation.

\item \textbf{Non-collapse property.}
The architecture preserves the individuality of schemas. Knowledge acquired in one schema does not automatically overwrite or absorb the contents of other schemas.

\item \textbf{Locality of learning.}
Knowledge updates are applied only to schemas participating in the active workflow, reducing interference between unrelated learning processes.

\item \textbf{Partial isolation of knowledge.}
Inactive schemas remain unaffected by local learning events occurring elsewhere in the knowledge structure, providing a degree of protection against catastrophic interference.

\item \textbf{Factored representation of the environment.}
The Mind operates on factored observation and decision spaces $(O,D)$ rather than monolithic representations. This decomposition allows the agent to organize information into multiple interacting subspaces and facilitates scalable reasoning.

While these spaces are constructed from multiple factors, this factorization does not imply statistical independence, informational orthogonality, or joint sufficiency. Different factors may contain overlapping, correlated, or redundant information. Consequently, the representation is factored but not necessarily factorial.

\item \textbf{Separation between knowledge use and knowledge update.}
The instantiation of schemas during reasoning is conceptually distinct from the processes that modify, refine, or create schemas. Knowledge execution and knowledge evolution therefore constitute separate architectural operations.

\item \textbf{Structural knowledge adaptation.}
The architecture allows not only parameter updates within existing schemas but also structural operations over the knowledge organization itself, including schema creation, differentiation, merging, specialization, abstraction, and refinement.

\item \textbf{Context-sensitive knowledge activation.}
Only a subset of schemas may become active in a given situation, enabling selective access to relevant knowledge without requiring the activation of the complete knowledge base.

\item \textbf{Knowledge reuse and transfer.}
Schemas learned in one context can be instantiated in different workflows and situations, supporting compositional reuse of previously acquired knowledge.

\item \textbf{Hierarchical knowledge organization.}
Knowledge may be organized across multiple levels of abstraction, allowing higher-level schemas to coordinate or summarize lower-level informational structures.

\item \textbf{Potential resistance to catastrophic forgetting.}
Since learning is localized and schemas remain partially isolated, modifications to one schema need not propagate globally through the entire knowledge structure. Although this does not eliminate catastrophic forgetting by itself, it provides architectural mechanisms that may mitigate it.

\item \textbf{Hierarchical and modular overfitting.}
The schema-based organization introduces the possibility that individual schemas, or entire schema hierarchies, become overspecialized to particular contexts. Such overfitting may occur locally even when the overall architecture retains significant generalization capabilities.

\end{itemize}

Persistent knowledge in SBL resides primarily in the evolving collection of schemas and their relations. Memory acts as a supporting mechanism for contextual retrieval, coordination, activation, and workflow construction, while the schemas themselves constitute the main long-term informational substrate of the architecture.

Table \ref{tab:arch_comparison} shows a brief informal comparison of the specific potential properties description for each architecture.

Figure \ref{fig:Representation ArchAgents} shows a representation of the topological space of ArchAgents that encapsulates the possible paths that go from RL to SBL, along with the different steps that we previously detailed. 

\begin{table}[!ht]
\centering
\renewcommand{\arraystretch}{1.25}
\begin{tabular}{lccc}
\hline
\textbf{Architectural dimension} & \textbf{RL} & \textbf{CRL} & \textbf{SBL} \\
\hline
Persistent information structure
    & Undifferentiated carrier $\Theta$
    & $(\Theta_\pi, \Theta_{CS})$
    & Family of schemas $\Sigma$ \\

Feedback structure
    & Single endomorphic loop
    & Two coupled endomorphic loops
    & Multiple decoupled loops \\

Causal structure
    & Not represented
    & Explicit causal model
    & Modular causal schemas \\

Information reuse
    & Not supported
    & Restricted
    & Compositional and reusable \\

Continual learning support
    & Limited
    & Partial
    & Architectural \\

Interface typing
    & Monolithic
    & Weakly typed
    & Strongly typed and factored \\

Body-Mind mediation
    & None
    & None
    & Explicit architectural layer \\

Locality of updates
    & Global
    & Role-based
    & Schema-local \\
\hline
\end{tabular}
\caption{Comparison of agent architectures as objects in $\mathbf{ArchAgents}$, highlighting differences in information and wiring structure, feedback, and modularity.}
\label{tab:arch_comparison}
\end{table}

\subsection{Case Study IV: AIXI Architecture}
AIXI \cite{Hutter2005,Hutter2024,Veness2011} can be represented in our framework as an architecture centered on
history-dependent universal prediction and Bayesian belief updating.
\subsubsection{AIXI syntax layer.}

The hypergraph category $Syn_{AIXI}$ is freely generated by the following
types symbols:
\begin{itemize}
    \item $O$ observations of the agent,
    \item $A$ actions that the agent can perform,
    \item $R$ reward received by the agent,
    \item $H$ historical memory of the agent,
    \item $\varepsilon$ stores the universal predictive kernel that, given the history and an action, generates a distribution over the next observations and rewards $P(\cdot|h,a)$.
\end{itemize}
All tensor expressions over these types are available via the symmetric
monoidal structure.
The primitive architectural generators of $Syn_{AIXI}$ include:
\[
\begin{aligned}
& InitEnvKernel: H \to \varepsilon \\
& Policy: H \otimes \varepsilon \to A \\
& EnvInteraction: H \otimes A \to O \otimes R\\
& UpdateHistory: H \otimes A \otimes O \otimes R \to H
\end{aligned}\]

These generators represent the following operations:
\begin{itemize}
    \item $InitEnvKernel$ is the operator for generating the kernel representing the universal bayesian prior inside $\varepsilon$,
    \item $Policy$ is the operator that generates the next action given the kernel and the history,
    \item $EnvInteraction$ is the operator that represents the interaction with the environment,
    \item $UpdateHistory$ is the operator for updating the history after interacting with the environment.
\end{itemize}

The syntax diagram $\mathcal{G}_{AIXI}$ in figure \ref{fig:AIXI string diagram} corresponds to initializing a universal predictive mixture, selecting an action from the current history and predictive model, interacting with the environment, and updating the history after each step.
The resulting workflow is therefore a closed sequential loop in which decision-making is mediated by a predictive model over future observation-reward continuations.

\begin{figure}[!ht]
    \centering
    \includegraphics[width=0.75\linewidth]{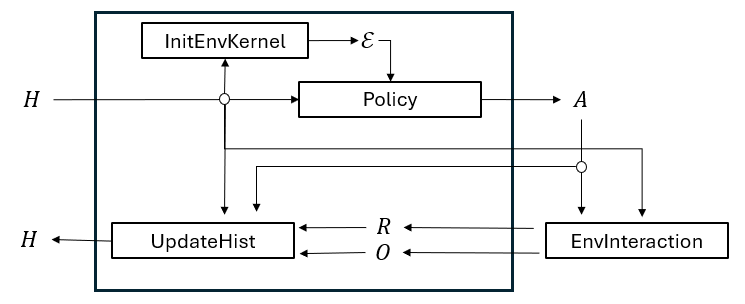}
    \caption{AIXI string diagram}
    \label{fig:AIXI string diagram}
\end{figure}

\subsubsection{AIXI knowledge layer.}

The knowledge architecture $Know_{AIXI}$ is freely generated by the knowledge presentation $K_{AIXI}=(KTypes_{AIXI},KGen_{AIXI},KEq_{AIXI})$ where. At the knowledge level, the architecture distinguishes between memory, hypothesis space, beliefs over hypotheses, and the induced predictive kernel:
\begin{itemize}
    \item $KTypes_{AIXI} = \{ M, K, \{E\}, W\}$, were $M$ is the knowledge unit representing the memory storing the history of the agent, $K$ is the knowledge unit representing the universal kernel, $\{E\}$ is the set of hypothesis about the potential next environments and $W$ is the unit of knowledge that represent the beliefs, that is, the weights over the possible set of environments.
    \item $KGen_{AIXI}$ is composed by the following transformations:
    \begin{itemize}
        \item $AggHist:M \to M$, is the operator that updates the history memory,
        \item $GenHypSpace: I \to \{E\}$, is the operator that generates the universal set of possible environments,
        \item $UniversalPrior:\{E\} \to W$, generates the distribution of beliefs over the environments hypothesis. 
        \item $PosteriorUpd: M \otimes W \to W$,
        \item $KernelMixing:\{E\}\otimes W \to K$, generates the universal predictive kernel from the beliefs and the hypothesis.
    \end{itemize}
    \item $KEq_{AIXI}$ is empty
\end{itemize}

This makes explicit that AIXI is not merely a policy loop, but an architecture whose behavior is mediated by an internal Bayesian epistemic layer. This architecture also shows that knowledge layer may contain internal epistemic structures that are not directly exposed in the syntax layer but that influence the agent behaviour through derived knowledge objects ($\{E\}, W$).

\begin{figure}
    \centering
    \includegraphics[width=0.8\linewidth]{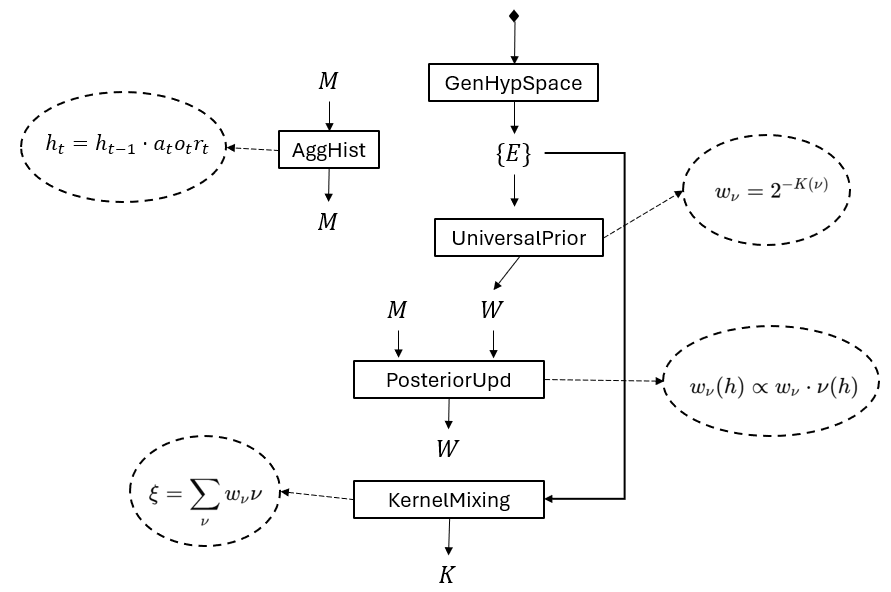}
    \caption{Theoretical AIXI Workflow knowledge and implementation of each operator}
    \label{fig:Theorical AIXI}
\end{figure}

\subsubsection{Relational Interface for AIXI.}
The interaction between the syntactic workflows and the internal knowledge resources in AIXI is specified by the architectural profunctor
\[ \Phi_{AIXI} : \mathcal{G}_{AIXI}^{\mathrm{op}} \nrightarrow \mathsf{Know}_{AIXI}\]

At the object level, this profunctor has non-empty support on the pairs
\[ (H,M), \qquad (\varepsilon,K),\]
reflecting that syntactic history is grounded in internal memory and that the predictive kernel used by the policy is generated from internal epistemic structure.
\begin{table}[h]
    \centering
    \renewcommand{\arraystretch}{1.25}
    \setlength{\tabcolsep}{18pt}
    
    \begin{tabular}{|c|c|c|c|c|c|}
    \hline
    \textbf{Types} & $O$ & $A$ & $R$ & $H$ & $\varepsilon$ \\
    \hline
    $M$ & \xmark & \xmark & \xmark & \cmark & \xmark\\
    \hline
    $K$ & \xmark & \xmark & \xmark & \xmark & \cmark\\
    \hline
    $\{E\}$ & \xmark & \xmark & \xmark & \xmark & \xmark \\
    \hline
    $W$ & \xmark & \xmark & \xmark & \xmark & \xmark \\
    \hline
    
    \end{tabular}
    \caption{Visualization as a table of the support of the AIXI relational interface profunctor}
\end{table}

\subsubsection{AIXI constraints}
This first presentation only captures the compositional skeleton of AIXI, not yet the full family of architectural laws that distinguish it from a generic Bayesian predictive agent.
In particular, standard AIXI carries stronger restrictions, namely:
\begin{enumerate}
    \item the hypothesis space should be admissibly interpretable as a universal class of computable or semi-computable environments, 
    \item the prior over hypotheses should satisfy a universal simplicity-weighting principle (Solomonoff/ Kolmogorov), 
    \item belief updates must be posterior-consistent with the accumulated interaction history,
    \item the predictive kernel must be induced as a Bayesian mixture over the weighted hypothesis class and 
    \item policy selection should be prospectively guided by expected future reward under that universal predictive mixture, typically through an expectimax-style sequential evaluation.
\end{enumerate}
 
Moreover, unlike standard Markovian RL, AIXI is fundamentally history-based, so its informational assumptions concern the sufficiency and admissibility of history representations rather than state-transition factorization.
These restrictions are essential for distinguishing AIXI from broader classes of model-based or Bayesian agents, but in the present paper they are not yet internalized as explicit architectural equations or constraint objects.

\subsubsection{AIXI as object of ArchAgents}

Finally, the AIXI architecture would be described as the following object in ArchAgents:
\[AIXI=(\mathcal{G}_{AIXI},Know_{AIXI},\Phi_{AIXI})\]

\subsubsection{Properties of the AIXI Architecture}

The AIXI architecture provides an interesting contrast with the previous architectures. While it possesses an extremely rich informational layer, its syntactic organization is comparatively simple and highly centralized. This distinction becomes explicit in the present framework.

\paragraph{Structural Properties}

Examples of structural properties that characterize the AIXI architecture include:

\begin{itemize}

\item \textbf{Centralized decision structure.}
The architecture is organized around a single global decision-making loop in which perception, prediction, planning and action selection are tightly coupled. Unlike modular architectures such as SBL, no explicit decomposition into specialized subsystems is required.

\item \textbf{Non-factorized syntactic workflow.}
The syntax diagram is essentially monolithic. Predictions, utility evaluation and action selection are performed within a unified expectimax computation rather than through independent compositional modules.

\item \textbf{Absence of explicit hierarchical decomposition.}
The architecture does not impose intermediate abstraction layers, hierarchical controllers, or compositional planning structures. All reasoning occurs within a single decision process.

\item \textbf{Global dependence of decisions.}
Action selection depends on the complete interaction history rather than on locally isolated syntactic components. Consequently, architectural modifications tend to affect the entire decision process.

\item \textbf{Low modularity.}
The architecture lacks explicit mechanisms for independent specialization, reuse, or recombination of syntactic subcomponents. This property contrasts strongly with architectures based on schemas, causal modules, or factorized world models.

\end{itemize}

\paragraph{Informational Properties}

The informational layer of AIXI is substantially richer and constitutes its defining characteristic.

\begin{itemize}

\item \textbf{Universal model class.}
The knowledge layer admits a class of computable environment hypotheses that is intended to be universal. In principle, any computable environment can be represented within the hypothesis space.

\item \textbf{Bayesian knowledge aggregation.}
Knowledge is represented through a Bayesian mixture over environment hypotheses. Learning corresponds to updating posterior weights after observing new interaction histories.

\item \textbf{Algorithmic prior structure.}
The architecture imposes a preference toward simpler hypotheses through a Solomonoff-style prior, introducing an explicit complexity bias into the organization of knowledge.

\item \textbf{Predictive knowledge orientation.}
Knowledge units are primarily organized around prediction of future observations and rewards rather than around causal, symbolic, or schema-based representations.

\item \textbf{Utility-guided inference.}
Knowledge transformations are coupled to long-horizon utility maximization. The informational layer is not merely predictive but directly supports expectimax planning over future trajectories.

\item \textbf{Global knowledge integration.}
All available evidence contributes to a single posterior distribution. The architecture does not require explicit partitioning of knowledge into independent modules or local models.

\end{itemize}

From the perspective of the comparative framework, AIXI can therefore be viewed as an architecture with relatively low structural modularity but extremely high informational expressivity. Its distinctive properties arise primarily from strong constraints imposed on the knowledge layer rather than from sophisticated syntactic organization.

%% file: Parts/7_Agents_implementations.tex
\section{Agents Implementations Sketches}
\subsection{A Concrete Implementation Sketch: Tabular RL in $\mathbf{Meas}$}
\label{app:RL_agent}
In this section, we present a concrete implementation of the Reinforcement Learning (RL) architecture within the proposed framework. The goal is to illustrate how a standard tabular $Q$-learning agent can be understood as a semantic realization of the abstract architecture.
We emphasize that this example is intentionally simple. Its purpose is not to showcase the expressive power of the framework, but to demonstrate how the formal notions of syntax, knowledge, and constraint satisfaction interact in a familiar setting.

\subsubsection{Choice of semantic category}

We work in the symmetric monoidal category $\mathbf{Meas}$, whose objects are measurable spaces and whose morphisms are measurable maps. The monoidal product is given by the cartesian product.
All spaces in this construction are finite or finite-dimensional, and are equipped with their discrete or Borel $\sigma$-algebras. In particular:
\begin{itemize}
    \item the state space $S$ and action space $A$ are finite measurable spaces;
    \item the parameter and knowledge spaces are finite-dimensional Euclidean spaces;
    \item all maps defined below are measurable.
\end{itemize}

Stochastic components (such as environment interaction) are described via Markov kernels. Formally, these may be interpreted either as morphisms in the Kleisli category of the Giry monad or extensionally via measurable mappings into spaces of probability measures. For simplicity, we use the kernel notation.

\subsubsection{Syntactic realization}

We now define the interpretation
\[
I : \mathcal{G}_{RL} \to \mathbf{Meas}.
\]

\paragraph{Types.}
\begin{itemize}
    \item $I(S)$: a finite measurable state space.
    \item $I(A)$: a finite measurable action space.
    \item $I(E) := S \times A \times R \times S$, where $R \subseteq \mathbb{R}$ is a measurable reward space.
    \item $I(\Theta^s) := \mathbb{R}^{S \times A}$.
\end{itemize}

An element $\theta \in I(\Theta^s)$ is a table assigning values:
\[
Q_\theta(s,a) \in \mathbb{R}.
\]

\paragraph{Generators.}

\begin{itemize}
    \item \textbf{Policy.}
    \[
    I(\mathrm{Policy}) : S \times \Theta^s \to A,
    \quad
    I(\mathrm{Policy})(s,\theta) \in \arg\max_{a \in A} Q_\theta(s,a).
    \]

    Since $A$ is finite, a measurable selection exists.

    \item \textbf{Environment interaction.}
    \[
    I(\mathrm{EnvInteraction}) : S \times A \rightsquigarrow E,
    \]
    defined by a Markov kernel
    \[
    P(r,s' \mid s,a),
    \]
    so that
    \[
    (s,a) \mapsto (s,a,r,s').
    \]

    \item \textbf{Update.}
    \[
    I(\mathrm{Update}) : \Theta^s \times E \to \Theta^s,
    \]
    defined by the tabular $Q$-learning rule. For $e=(s,a,r,s')$:
    \[
    Q'(x,u)=
    \begin{cases}
    Q(s,a) + \alpha \left(r + \gamma \max_{u'} Q(s',u') - Q(s,a)\right),
    & (x,u)=(s,a),\\[4pt]
    Q(x,u), & \text{otherwise.}
    \end{cases}
    \]
\end{itemize}

\subsubsection{Knowledge realization}

We now define
\[
J : Know_{RL} \to \mathbf{Meas}.
\]

\paragraph{Types.}

\begin{itemize}
    \item $J(\Theta^k) := \mathbb{R}^{S \times A}$,
    interpreted as the space of action-value functions. This interpretation is valid since S and A are discrete spaces, and therefore $\mathbb{R}^{S \times A} = \{Q:S\times A \to \mathbb{R}\}$
    \item $J(E^k) := S \times A \times R \times S$.
\end{itemize}

\paragraph{Knowledge update.}

\[
J(\mathrm{Upd}) : \Theta^k \times E^k \to \Theta^k
\]

is defined by the same TD update rule:
\[
Q'(s,a)= Q(s,a) + \alpha (r + \gamma \max_{u} Q(s',u)-Q(s,a)).
\]

\begin{remark}[Degenerate syntax--knowledge alignment]
In this tabular realization,
\[
I(\Theta^s) \cong J(\Theta^k).
\]
Thus, the syntactic parameter carrier and the knowledge carrier coincide extensionally.

This reflects a structural degeneracy of tabular RL, where representation and represented knowledge collapse into the same object. In more expressive realizations (e.g., neural agents), this identification does not necessarily hold.
\end{remark}

\subsubsection{Syntax--knowledge compatibility}

The relational interface $\Phi_{RL}$ is supported on $(\Theta^s,\Theta^k)$.
We define the correspondence
\[
R : I(\Theta^s) \to J(\Theta^k)
\]
as the identity:
\[
R(\theta) = Q_\theta.
\]

The only knowledge-compatible syntactic generator is $\mathrm{Update}$.
Let us define:
\[
k_{\mathrm{Update}} := \mathrm{Upd} : \Theta^k \times E^k \to \Theta^k.
\]

Then, the compatibility condition of Definition 4.1.2 is satisfied:
\[
R(I(\mathrm{Update})(\theta,e))
=
J(\mathrm{Upd})(R(\theta),e).
\]

Thus, $I(\mathrm{Update})$ is induced by $J(\mathrm{Upd})$.

\subsubsection{Satisfaction of the constraint layer}

We now verify that this realization satisfies the constraint layer.
Each constraint is interpreted as a restriction on realizations.

\paragraph{(1) Value representability.}

\[
\rho^{RL}_{val}
=
(\Theta^k, \mathbb{R}^{S \times A}).
\]

This is satisfied by construction.

\paragraph{(2) Bellman consistency.}

Let $\mathcal{B}$ be the Bellman optimality operator:
\[
(\mathcal{B}Q)(s,a)
=
\mathbb{E}[r + \gamma \max_{a'} Q(s',a') \mid s,a].
\]

Define:
\[
\rho^{RL}_{Bell}
=
(\Theta^k, \{ Q \mid Q = \mathcal{B}Q \}).
\]

The update $J(\mathrm{Upd})$ defines a stochastic approximation process converging to this set under standard assumptions. Hence, the realization is Bellman-compatible.

\paragraph{(3) Policy--value compatibility.}

\[
\rho^{RL}_{pol}
=
(\mathrm{Policy}, \{ \pi \mid \pi(s) \in \arg\max_a Q(s,a) \}).
\]

This holds by construction of $I(\mathrm{Policy})$.

\paragraph{(4) Markov admissibility.}

\[
\rho^{RL}_{Markov}
=
(\mathrm{EnvInteraction}, \{ P \mid P(e \mid history) = P(e \mid s,a) \}).
\]

This is satisfied by the definition of the transition kernel.

\subsubsection{Conclusion}

We obtain a realization
\[
F^{RL}_{tab} = (I,J)
\]
in $\mathbf{Meas}$ which:

\begin{itemize}
    \item interprets the syntactic structure as a concrete learning system,
    \item interprets the knowledge layer as action-value functions,
    \item satisfies the syntax--knowledge compatibility condition,
    \item and fulfills the RL constraint layer.
\end{itemize}

Thus,
\[
F^{RL}_{tab} \models R_{RL},
\]
and constitutes an admissible agent of the enriched RL architecture.

This example illustrates how a classical tabular $Q$-learning agent fits naturally within the proposed framework, while also highlighting that in simple cases the distinction between syntax and knowledge may collapse.
Conceptually, this example illustrates the role of the framework very clearly. The RL architecture itself does not specify $Q$-learning, nor argmax policy, nor a tabular representation. It only specifies the admissible compositional organization and the relevant architectural commitments. The concrete choices made here belong to the implementation layer.
This distinction is precisely what allows the same architecture to admit many different agents: tabular agents, neural agents, model-based agents, actor-critic agents, and so on. All of them potentially realizing the same abstract RL architectural object under different semantic interpretations.

\subsubsection{Neural parameterization of the syntactic carrier}

We briefly describe how the previous tabular realization can be modified by replacing the syntactic parameter carrier with a neural network, while leaving the knowledge layer unchanged.

\paragraph{Syntactic carrier.}

We define $I(\Theta^s)$ as the parameter space of a three-layer feedforward neural network with one hidden layer of width $k$.

Let the input be a representation of $(s,a) \in S \times A$ in $\mathbb{R}^d$. The network is defined by:
\[
\begin{aligned}
h &= \sigma(W_1 x + b_1), \quad W_1 \in \mathbb{R}^{k \times d}, \; b_1 \in \mathbb{R}^k, \\
Q_\theta(s,a) &= W_2 h + b_2, \quad W_2 \in \mathbb{R}^{1 \times k}, \; b_2 \in \mathbb{R},
\end{aligned}
\]
where $\sigma$ is a fixed activation function (e.g., ReLU), and $\theta = (W_1,b_1,W_2,b_2)$.

Thus:
\[
I(\Theta^s) := \mathbb{R}^{k \times d} \times \mathbb{R}^k \times \mathbb{R}^{1 \times k} \times \mathbb{R}.
\]

\paragraph{Induced action-value function.}

Each $\theta \in I(\Theta^s)$ induces:
\[
Q_\theta : S \times A \to \mathbb{R},
\quad
Q_\theta(s,a) := \mathrm{NN}_\theta(s,a).
\]

\paragraph{Update.}

Given an experience $e=(s,a,r,s')$, define the TD target:
\[
y = r + \gamma \max_{u} Q_\theta(s',u),
\]
and the loss:
\[
\mathcal{L}(\theta;e) = \big(Q_\theta(s,a) - y\big)^2.
\]

The syntactic update is given by stochastic gradient descent:
\[
I(\mathrm{Update})(\theta,e)
=
\theta - \alpha \nabla_\theta \mathcal{L}(\theta;e).
\]

\paragraph{Knowledge layer.}

The knowledge realization remains unchanged:
\[
J(\Theta^k) = \mathbb{R}^{S \times A},
\quad
J(\mathrm{Upd}) \text{ given by the TD rule.}
\]

Thus, the intended transformation at the level of action-value functions is identical to the tabular case.

\paragraph{Interpretation.}

The key difference lies in the implementation mechanism. In the tabular case, the syntactic update directly realizes the knowledge-level transformation. In the neural case, the update operates on parameters and only indirectly affects the represented function.

This illustrates that different syntactic realizations may implement the same knowledge-level transformation through substantially different mechanisms, while preserving the overall architectural structure.

\begin{remark}[On weakened compatibility and neural realizations]
In the neural setting, the compatibility condition of Definition 4.1.2 is generally not satisfied strictly. In particular, although the syntactic update is designed to approximate the knowledge-level transformation, one typically has
\[
R(I(\mathrm{Update})(\theta,e)) \neq J(\mathrm{Upd})(R(\theta),e).
\]

Nevertheless, there exists a corresponding knowledge-level generator
\[
k_{\mathrm{Update}} = \mathrm{Upd} : \Theta^k \times E^k \to \Theta^k,
\]
which captures the intended transformation at the level of action-value functions.

This suggests a weaker notion of compatibility, in which admissibility is determined by the existence of a corresponding knowledge-level transformation explaining the syntactic behavior, rather than by exact equality of realizations.

Formalizing such a weakened notion of admissibility, possibly in terms of approximation or constraint satisfaction, remains an open direction and highlights a limitation of the current framework when applied to function approximation settings.
\end{remark}

%% file: Parts/Work_in_Progress.tex
\section{Work in Progress and Future Research Directions}
\label{sec:future}

The present work should be understood as an initial organizing step towards a category-theoretic framework for the comparative analysis of AGI architectures. In particular, Sections \ref{sect:ArchAgentCat} -\ref{sect:properties} introduce the core architectural, implementation and property-based layers. The following tasks extend these foundations in a controlled and incremental manner. We organize them according to increasing temporal horizon and conceptual scope.

\subsection{Very Short-Term Extensions}

\begin{itemize}
    
    \item \textbf{Development of Theoretical Results and Structural Theorems} The current paper introduces the categorical machinery required to compare AGI architectures, but does not yet fully exploit this machinery to derive non-trivial mathematical results. Therefore, a primary short-term objective is the development and proof of the structural theorems outlined in Appendix \ref{app:Foundational theorems}. These include results concerning architectural equivalence, admissibility preservation, realizability, expressivity hierarchies, and architecture–agent correspondences. Establishing such results is essential for demonstrating that the framework is not merely a descriptive language, but a genuine mathematical theory capable of deriving formal guarantees and comparative insights about different AGI architectures.
    
    \item \textbf{Better formalization of the relational interface}
    The syntax--knowledge interface is modeled via profunctors, which may be overly expressive and lack canonical instantiations; we are investigating simpler alternatives (e.g., Boolean relations or optics). 

    \item \textbf{Explicit specification of architectural morphisms.}
    The informal notion of translation between architectures introduced in Section \ref{sect:ArchAgentCat} can be refined by explicitly stating which generators, wiring constraints and informational access patterns are preserved or weakened by morphisms. Although morphisms in $\mathbf{ArchAgents}$ capture structural transformations, there is no theory yet describing property preservation or transfer across architectures. Developing this connection is essential for turning the framework into a predictive and comparative tool.

    \item \textbf{Deeper analysis of architectural properties.}
    Sections \ref{sect:properties} and \ref{sect:CaseStudies} motivate the distinction between structural, informational and semantic properties. A short-term task is to further formalize these classes of properties and their logical dependencies.

    \item \textbf{Property inheritance from architectures to agents.}
    As suggested in Section \ref{sect:CaseStudies}, structural or informational properties of architectures can be interpreted as hypotheses inherited by agents implementing those architectures.
    
\end{itemize}

\subsection{Short-Term Extensions}

\begin{itemize}
    
    \item \textbf{A category of properties.}
    Building on Section \ref{sect:properties}, properties can be organized into a category or preorder, where $P \le_A Q$ denotes abstraction or implication between properties relative to an architecture $A$.

    \item \textbf{Additional illustrative examples.}
    The architectural examples of Section \ref{sect:CaseStudies} can be extended to include HRL, MMRL, and active inference, as well as multiple agents sharing the same architecture but differing in their implementations.

\end{itemize}

\subsection{Mid-Term Extensions}

These directions introduce new categorical layers that connect the theoretical framework with empirical evaluation.

\begin{itemize}

    \item \textbf{Learning operators at the informational level.}
    Building on the treatment of information flow in Section \ref{sect:ArchAgentCat}, learning and update mechanisms can be formalized as structured endomorphisms (or optics/lenses) on internal state objects, making explicit the distinction between architectural learning capacity and concrete learning algorithms.

    \item \textbf{Environment and World abstraction.}

    While the this paper focuses on studying agents and architectures theoretically, it is also necessary/essential to formalize the empirical evaluation and comparison of agents. This requires an explicit notion of environments and world. A new categorical layer \textbf{World} can be introduced to model agent-environment interactions and controlled interventions, and will also enable the use of measurement instruments or realization of experiments that let users measure the performance of agents, or even alter the environments.

    This involves the following formalizations:
    \begin{itemize}
        \item \textbf{Categorical modeling of interactions.}: Objects of \textbf{World} can be defined as triples $(w,a,e)$ representing concrete interaction instances, with morphisms capturing transformations or evolution of interaction dynamics, extending the agent-environment coupling discussed in Section \ref{sect:ArchAgentCat}.
        
        \item \textbf{Higher-categorical interaction structure.}: World could also be seen as n-category where 0-morphisms are $(w,a,e)$, and 1-morphisms correspond to the interactions between agents and environments $(w,a,e)\to (w',a',e')$, etc. We can even have n-morphisms representing the evolution of a solely agent or the interactions between different agents: $(w^1 ,a_1,e) \to (w^1,a_1,e') \to (w^2,a'_2,e')\to \dots$. 
    
        Iterated or parallel interactions motivate higher morphisms representing
        evolving agents, environments or multi-agent scenarios.

        \item \textbf{Empirical measurement and testing.}:
        Building on the distinction between theoretical properties and empirical
        evaluation introduced in Section \ref{sect:properties}, experimental
        descriptions and testers can be defined to map interaction data to empirical scores, clearly separating logical comparison from benchmark-based evaluation.
        
        To empirically evaluate agents and measure the degree of its properties, we need a formalization that enable to describe the settings of the experiments in the world, how the environment will work, the length $n$ of the maximum number of interaction records, the external influence in the behaviour of agents. 
        It is also needed an object that behaves like a tester that, given the data from an experiment, maps it to the space of scores.
        Scores would be empirical evaluations of an agent's properties, that depend on the environment, dynamics and definition of the experiment. It does not depend on the environments, nor the experiments.

        \item \textbf{Formalization of environments.}
        The formalization of the World implies the existence of environments. Environments can be structured analogously to agents, potentially introducing architectural descriptions (\textbf{ArchEnv}), implementations (\textbf{Env}) and environment-level properties (pomdp-ness, causal smoothness, latent structure variability, stochasticity, etc), extending the symmetry suggested throughout the paper.

    \end{itemize}

    \item \textbf{Environment-dependent properties.}
    Section \ref{sect:properties} already hints at the role of environmental assumptions. An extension is to explicitly characterize properties whose validity depends on hypotheses about the environment or the agent-environment coupling.

    \item \textbf{Functorial comparison with related frameworks.}
    Functorial comparison with related frameworks: The relationship between $\mathbf{ArchAgents}$ and other categorical approaches to agency and control, such as categorical cybernetics, operadic wiring diagrams or optics-based formulations of reinforcement learning, can be made explicit via functors or adjunctions. This would clarify the scope and limitations of the proposed framework.

    The comparative perspective introduced in Section \ref{sect:CaseStudies} can be extended by making explicit the functorial relationships between $\mathbf{ArchAgents}$ and other categorical approaches to agency and control, such as categorical cybernetics or optics-based RL.
\end{itemize}

\subsection{Long-Term Extensions}

These directions correspond to the full maturation of the framework.

\begin{itemize}
    \item \textbf{Architectural design principles for AGI.}
    Architectural design principles for AGI: Ultimately, the framework aims not only to compare existing architectures, but also to identify structural principles necessary or sufficient for general intelligence. At this stage, the framework would function as a genuine theory of AGI architectures rather than a purely descriptive taxonomy.
    
    Synthesizing the results of Sections \ref{sect:ArchAgentCat}-\ref{sect:CaseStudies}, the framework aims to identify structural principles that are necessary or sufficient for general intelligence.

    \item \textbf{A unifying comparative theory of AGI architectures.}
    The long-term objective is to organize architectures within a category or lattice equipped with invariants and comparison principles, enabling rigorous statements about subsumption and equivalence.

    \item \textbf{Architectural expressivity.}
    Motivated by the comparative goals stated in Section \ref{sect:introduction} and Section \ref{sect:CaseStudies} a central open problem is the comparison of architectures in terms of the class of agents or behaviors they can express. This suggests defining suitable notions of simulation, expressivity or preorder relations between architectures, and studying whether certain architectures are universal or conservative extensions of others.

    \item \textbf{Collaborative experimental platform.}
    To operationalize the framework, a collaborative platform can be developed where architectures, agents and properties are shared, compared and evaluated through controlled experiments in common environments.
    
\end{itemize}

%% file: Parts/Conclusion.tex
\section{Conclusions}

\paragraph{Conclusions.}
Advancing towards Artificial General Intelligence requires not only improved algorithms, but a principled understanding of the structural organization of intelligent systems. In this work, we have proposed a category-theoretic framework that separates syntax, knowledge, and constraints, enabling a compositional and architecture-level analysis of agent designs.

This paper proposes an AGI comparative framework based on Category Theory. Rather than viewing an architecture as a concrete algorithm, we treat it as a structured theory of computational interconnections: a specification of admissible interfaces, primitive components, and compositional wiring patterns. This shifts the focus from implementation details to structural organization. Crucially, we distinguish two layers that are often conflated: on one hand, the \emph{syntactic architecture layer}, which governs how modules may be composed, and the \emph{knowledge management layer}, on the other hand, which governs how information is represented, transformed, and reused within that structure. Architectures may exhibit similar module flows while differing fundamentally in how they encapsulate models, aggregate evidence, or modularize experience. Thus, making this separation explicit is essential for identifying genuine structural differences and formally characterizing architectural properties.
Moreover, we add the constraint layer, which enables us to take into consideration the relationship between syntax and knowledge, and the fact that architectures also impose some theoretical axioms or constraints over the possible implementations.

\begin{wrapfigure}{r}{0.5\textwidth}
    \includegraphics[width=\linewidth]{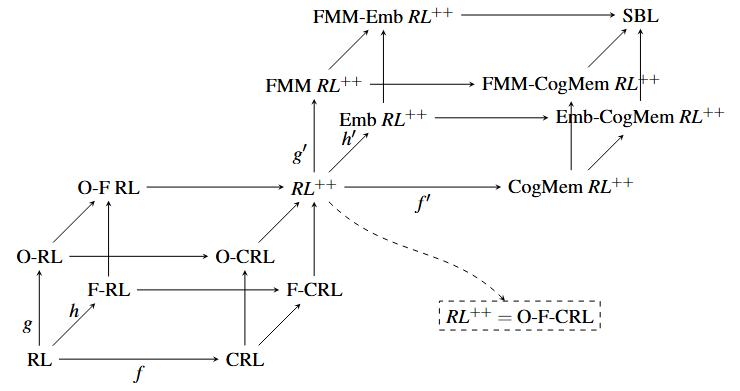}
    \caption[]{Example paths in the topological space of architectures.\footnotemark }
    \label{fig:Representation ArchAgents}
\end{wrapfigure}

\footnotetext{Morphisms denote architectural transformations composed in two stages: first among $f,g,h$, then among $f',g',h'$. Prefixes indicate architectural variants ($O$ offline, $F$ factorized, $FMM$ factored multimodel, $CogMem$ cognitive modules with memory augmented and $Emb$ embodied).}

The framework provides a first step towards a unifying formalism in which diverse AGI paradigms can be expressed and compared within a common mathematical setting. Crucially, it suggests that architectures themselves can be studied as elements of a structured space, where morphisms capture systematic transformations between them. As illustrated in Figure~\ref{fig:Representation ArchAgents}, this perspective opens the possibility of reasoning about paths, equivalences, and transformations in the space of architectures, rather than analyzing each paradigm in isolation.
We argue that such a unifying, category-theoretic perspective is not merely descriptive, but necessary for progressing towards AGI, as it enables the identification of structural principles, transfer of insights across paradigms, and a systematic exploration of the design space of intelligent agents.

The broader research program underlying this work seeks to provide a unified formal foundation for AGI systems, integrating architectural structure, informational organization, semantic/agent realization, agent–environment interaction, behavioural development over time, and the empirical evaluation of properties. This framework is also intended to support the definition of architectural properties, both syntactic and informational, as well as semantic properties of agents and their assessment in environments with explicitly characterized features.
We claim that Category Theory and AGI will have a very \textit{symbiotic relation}. That is, AGI will immensely benefit from a Category-theoretic general formalization, while, at the same time, Category Theory will become the front line mathematical paradigm thanks to the extremely wide interest in AGI.\\

\section*{Acknowledgments}
We would like to acknowledge Christoph von der Malsburg for his very inspiring conversations about SBL and Brain Theory in general. FC would also like to acknowledge Paul S. Rosenbloom for very enlightening discussions on Cognitive Architectures and Wei-Min Shen for  inspiring conversations on autonomous learning in agents and robots. 

\paragraph{Funding}
This research was supported by Cognodata Consulting SL.

\paragraph{Conflicts of Interest}
Pablo de los Riscos was employed by the company Cognodata Consulting

%% file: Parts/Framework_Theorems.tex
\section{Framework Theorems and Results}
\label{app:Foundational theorems}
\subsection{Foundational Structural Results}

\begin{theorem}[Constraint Preservation under Admissibility-Preserving Architecture Morphisms]
Let
\[F : A \to B\]
be an architecture morphism between agent architectures.

Assume that:
\begin{enumerate}
    \item $F$ preserves distinguished constraints, i.e.
    \[F^\sharp(R_A) \subseteq R_B,\]
    
    \item the induced scope translation functor
    \[F_C : C_A \to C_B\]
    preserves admissible reindexing of constraints.
\end{enumerate}

Then every admissible realization
\[(I,J) \in \mathrm{Agents}(A,E)\]
induces an admissible realization
\[F_*(I,J) \in \mathrm{Agents}(B,E).\]

Equivalently, architecture morphisms preserving constraints transport admissible agents to admissible agents.
\end{theorem}

This theorem formalizes when an architecture may be refined, compiled, abstracted, or translated without violating its defining admissibility conditions.

\begin{theorem}[Architecture Equivalence Implies Agent Equivalence]
Let $A$ and $B$ be architectures in $\mathrm{ArchAgents}$.

Suppose there exists an architectural equivalence
\[F:A\simeq B,\]
consisting of:

consisting of:
\begin{itemize}
    \item symmetric monoidal equivalences
    \[
    F_G : G_A \simeq G_B,
    \qquad
    F_{\mathrm{Know}} : \mathrm{Know}_A \simeq \mathrm{Know}_B,
    \]
    
    \item compatible equivalences of relational interfaces,
    
    \item and equivalences of constraint fibrations preserving distinguished constraints,
\end{itemize}

then for every semantic universe $E$,
the induced categories of admissible realizations are equivalent:
\[
\mathrm{Agents}(A,E)
\simeq
\mathrm{Agents}(B,E).
\]
\end{theorem}

Two apparently different AGI architectures may realize exactly the same class of admissible agents up to categorical equivalence.

\begin{theorem}[RL as a Degenerate Case of CRL]
Let

\[A_{\mathrm{CRL}}\]

be the causal reinforcement learning architecture.
Assume there exists an architecture morphism

\[D:A_{\mathrm{CRL}} \to A_{\mathrm{RL}}\]
that:

\begin{enumerate}
\item forgets causal intervention structure,

\item identifies causal dependencies with observational transition dependencies,

\item and maps causal constraints to Bellman-type consistency constraints.
\end{enumerate}

Then $D$ induces an architectural degeneration

\[A_{\mathrm{CRL}} \twoheadrightarrow A_{\mathrm{RL}},\]

and the resulting quotient architecture is equivalent to

\[A_{\mathrm{RL}}.\]
\end{theorem}

Classical reinforcement learning appears as a structural degeneration of causal reinforcement learning when explicit causal organization disappears.


\begin{theorem}[Strict Structural Inclusion $RL \subsetneq CRL$]
There exists a faithful architecture embedding
\[\iota : A_{\mathrm{RL}} \hookrightarrow A_{\mathrm{CRL}}\]
preserving:
\begin{itemize}
    \item syntactic workflows,
    \item admissible realizations,
    \item and Bellman-type constraints.
\end{itemize}

However, no inverse equivalence exists:
\[ A_{\mathrm{RL}} \not\simeq A_{\mathrm{CRL}}. \]

In particular, there exist causal workflows and admissible knowledge transformations in
\[ A_{\mathrm{CRL}} \]
which cannot be represented in
\[ A_{\mathrm{RL}}. \]

Therefore:
\[ A_{\mathrm{RL}} \subsetneq A_{\mathrm{CRL}} \]

\end{theorem}

CRL possesses strictly greater representational and architectural expressivity than standard RL.


\begin{theorem}[AIXI Non-Factorizability]
Let
\[ A_{\mathrm{AIXI}}\]
be an AIXI-like architecture.

Assume that:
\begin{enumerate}
    \item policy evaluation depends on universal environment induction,
    
    \item action selection is performed through a globally coupled expectimax computation,
    
    \item admissible realizations satisfy universal Bayesian consistency constraints.
\end{enumerate}

Then there does not exist a non-trivial decomposition
\[A_{\mathrm{AIXI}} \simeq A_1 \otimes A_2\]
such that:
\begin{itemize}
    \item knowledge workflows factor independently,
    
    \item distinguished constraints decompose componentwise,
    
    \item admissible realizations factor as tensor products.
\end{itemize}

Hence AIXI is structurally non-factorizable with respect to modular architectural decomposition.
\end{theorem}

This theorem formalizes mathematically why AIXI behaves as a globally coupled architecture rather than a modular cognitive system.

\begin{theorem}[Architecture-Implementation Adjunction]
There exists an adjunction
\[\mathcal{R} \dashv \mathcal{U}\]
between:
\begin{itemize}
    \item an architecture realization functor
    \[
    \mathcal{R} : \mathrm{ArchAgents} \to \mathrm{Impl},
    \]
    
    \item and an architectural abstraction functor
    \[
    \mathcal{U} : \mathrm{Impl} \to \mathrm{ArchAgents}.
    \]
\end{itemize}

Such that:
\[
\mathrm{Hom}_{\mathrm{Impl}}(\mathcal{R}(A),X)
\cong
\mathrm{Hom}_{\mathrm{ArchAgents}}(A,\mathcal{U}(X)).
\]

Naturally in both $A$ and X.
\end{theorem}

This theorem formalizes the duality between abstract architectural specification and concrete semantic implementation.

\begin{theorem}[Architectural Yoneda Principle]
Let
\[X,Y\in G_A.\]

Then

\[X\cong Y\]

if and only if

\[\mathrm{Hom}_{G_A}(-,X) \cong \mathrm{Hom}_{G_A}(-,Y)\]

naturally.

Consequently, an architectural component is completely characterized by its compositional interaction behaviour within the architecture.
\end{theorem}

This theorem gives a rigorous mathematical basis to the idea that cognitive functionality emerges from compositional interaction structure.

\subsection{Compositionality and Representability}

\begin{theorem}[Compositional Realizability]
Let
\[ (I_A,J_A)\in\mathrm{Agents}(A,E),
\qquad
(I_B,J_B)\in\mathrm{Agents}(B,E)
\]

Assume that the distinguished constraints of $A$ and $B$ are compatible under the monoidal product architecture.

Then
\[(I_A\otimes I_B, J_A\otimes J_B)\]
defines an admissible realization of

\[A\otimes B.\]

Consequently,

\[ \mathrm{Agents}(A,E)\times
\mathrm{Agents}(B,E)
\hookrightarrow
\mathrm{Agents}(A\otimes B,E).
\]
\end{theorem}

This theorem justifies modular AGI construction mathematically.

\begin{theorem}[Functoriality of Agent Realization]
For every semantic universe $E$, the assignment
\[A \longmapsto \mathrm{Agents}(A,E)\]

extends to a functor

\[ \mathrm{Agents}_E : \mathrm{ArchAgents}^{\mathrm{adm}} \longrightarrow \mathbf{Cat},\]

where \(\mathrm{ArchAgents}^{\mathrm{adm}}\) denotes the subcategory whose morphisms preserve distinguished constraints.

In particular, every admissibility-preserving architecture morphism

\[F:A\to B\]
induces a reindexing functor
\[F^*: \mathrm{Agents}(B,E) \to \mathrm{Agents}(A,E).\]
\end{theorem}

This turns the entire space of AGI architectures into a mathematically transformable object.

\begin{theorem}[Constraint Pullback Stability]
Let
\[\pi_A:P_A\to C_A\]

be the architectural constraint fibration. For every scope morphism
\[f:X\to Y\]

in $C_A$, every constraint
\[P\in P_A(Y)\]
admits a cartesian pullback constraint

\[f^*P\in P_A(X).\]

Moreover, satisfiability is preserved under pullback along admissible realizations.
\end{theorem}

Constraints propagate coherently across hierarchical architectures.

\begin{theorem}[Constraint Satisfaction Invariance]
Let
\[(I,J)\cong(I',J')\]

be isomorphic admissible realizations of an architecture $A$.

Then for every distinguished constraint
\[\rho\in R_A,\]

\[(I,J)\models\rho \quad\Longleftrightarrow\quad (I',J')\models\rho.\]

Hence, constraint satisfaction depends only on the isomorphism class of a realization.
\end{theorem}

Architectural properties do not depend on accidental implementation details.

\begin{theorem}[Representable Syntax--Knowledge Interfaces]
Let
\[\Phi_A : G_A \rightsquigarrow \mathrm{Know}_A\]
be the relational interface of an architecture.

If $\Phi_A$ is representable in the knowledge argument, then for every object
\[X\in G_A\]

there exists an object
\[K_X \in \mathrm{Know}_A\]
such that:
\[\Phi_A(X,-) \cong \mathrm{Hom}_{\mathrm{Know}_A}(K_X,-)\]

naturally in $\mathrm{Know}_A$.

The object $K_X$ acts as a universal internal representation associated with the syntactic interface $X$.
\end{theorem}

Every syntactic interface admits a universal internal cognitive representation.

\begin{corollary}[Operational Indistinguishability]
Two architectural components are operationally equivalent iff they are indistinguishable under all admissible workflows.

Formally:
\[X \cong Y \quad\Longleftrightarrow\quad \mathrm{Hom}_{G_A}(-,X)
\cong
\mathrm{Hom}_{G_A}(-,Y).
\]
\end{corollary}

\subsection{Expressivity and Universality}

\begin{theorem}[Architectural Expressivity Hierarchy]
Define

\[[A]\preceq[B]\]

for architecture equivalence classes whenever there exists a faithful architecture embedding

\[A\hookrightarrow B.\]

Then $\preceq$ defines a partial order on equivalence classes of architectures.

This order induces an architectural expressivity hierarchy over $\mathrm{ArchAgents}$.
\end{theorem}

Provides a mathematical notion of ``architectural power''.

\begin{theorem}[Universal Simulation of Architectures]
There exist architectures
\[U \in \mathrm{ArchAgents}\]
such that every architecture
\[A\]
admits a realization-preserving embedding
\[A \hookrightarrow U.\]
\end{theorem}

Defines a categorical notion of universal AGI architecture.

\begin{theorem}[Existence of Semantic Realizations]
Every consistent architecture
\[A\]
admits a semantic realization in some symmetric monoidal category $E$:
\[\mathrm{Agents}(A,E)\neq \emptyset.\]
\end{theorem}

Guarantees that the framework does not generate empty or unrealizable architectures.

\begin{theorem}[Minimal Knowledge Realization]
For every architecture $A$, there exists an admissible realization minimizing knowledge complexity with respect to a designated complexity ordering:
\[(I_{\min},J_{\min}).\]
\end{theorem}

Introduces a formal notion of minimal cognitive realization.

\begin{theorem}[Operational Equivalence of Syntax Interfaces]
Let
\[X,Y \in \mathrm{Syn}_A.\]

If
\[\mathrm{Hom}_{\mathrm{Syn}_A}(-,X)\cong
\mathrm{Hom}_{\mathrm{Syn}_A}(-,Y)
\]
naturally, then
\[X \cong Y.\]
\end{theorem}

Two syntax interfaces are equivalent iff they admit exactly the same compositional interaction patterns.